\documentclass[nohyperref]{article}

\usepackage{microtype}
\usepackage{graphicx}
\usepackage{subfigure}
\usepackage{booktabs} %

\usepackage[dvipsnames]{xcolor}
\usepackage[colorlinks=true,linkcolor=BrickRed, urlcolor=Blue,,citecolor=Blue,anchorcolor=blue, backref=page]{hyperref}
\renewcommand*{\backref}[1]{}
\renewcommand*{\backrefalt}[4]{%
    \ifcase #1%
          \or [Cited on page~#2.]%
          \else [Cited on pages~#2.]%
    \fi%
    }

\usepackage[accepted]{icml2023}
\usepackage{notation}

\usepackage{amsmath}
\usepackage{amssymb}
\usepackage{mathtools}
\usepackage{amsthm}
\usepackage{thmtools, thm-restate}
\usepackage{multirow}

\usepackage{enumitem}
\usepackage[capitalize,noabbrev]{cleveref}
\crefname{figure}{Fig.}{Figs.}
\crefname{definition}{Defn.}{Defns.}
\crefname{corollary}{Cor.}{Cors.}
\crefname{proposition}{Prop.}{Props.}
\crefname{theorem}{Thm.}{Thms.}
\crefname{remark}{Remark}{Remarks}
\crefname{principle}{Principle}{Principles}
\crefname{lemma}{Lemma}{Lemmata}
\crefname{claim}{Claim}{Claims}
\crefname{table}{Tab.}{Tabs.}
\crefname{section}{\S}{\S\S}
\crefname{subsection}{\S}{\S\S}
\crefname{subsubsection}{\S}{\S\S}
\crefname{assumption}{Assumption}{Assumptions}
\crefname{appendix}{Appx.}{Appx.}
\crefname{equation}{Eq.}{Eqs.}
\crefname{example}{Example}{Examples}

\usepackage{titlesec}
\titlespacing{\paragraph}{%
  0pt}{%
  0pt}{%
  1em}%

\everypar{\looseness=-1}

\icmltitlerunning{Provably Learning Object-Centric Representations}

\begin{document}

\twocolumn[
\icmltitle{Provably Learning Object-Centric Representations}

\icmlsetsymbol{equal}{*}
\icmlsetsymbol{last}{$\dagger$}

\begin{icmlauthorlist}
\icmlauthor{Jack Brady}{equal,mpi}
\icmlauthor{Roland S. Zimmermann}{equal,mpi}
\icmlauthor{Yash Sharma}{mpi}
\icmlauthor{Bernhard Sch\"olkopf}{mpi}\\
\icmlauthor{Julius von K\"ugelgen}{last,mpi,cam}
\icmlauthor{Wieland Brendel}{last,mpi}
\end{icmlauthorlist}

\icmlaffiliation{mpi}{Max Planck Institute for Intelligent Systems, T\"ubingen, Germany}
\icmlaffiliation{cam}{Department of Engineering, University of Cambridge, Cambridge, United Kingdom}

\icmlcorrespondingauthor{Jack Brady, Wieland Brendel}{first.last@tue.mpg.de}

\icmlkeywords{Identifiability theory, object-centric learning, latent variable model, unsupervised learning}

\vskip 0.3in
]

\printAffiliationsAndNotice{\textsuperscript{*}Equal contribution  \textsuperscript{$\dagger$}Shared last author} %

\begin{abstract}%
Learning structured representations of the visual world in terms of objects promises to significantly improve the generalization abilities of current machine learning models. While recent efforts to this end 
have shown promising empirical progress, a theoretical account of when unsupervised object-centric representation learning is possible is still lacking. Consequently, understanding the reasons for the success of existing object-centric methods as well as designing new theoretically grounded methods remains challenging. In the present work, we analyze when object-centric representations can provably be learned without supervision. To this end, we first introduce two assumptions on the generative process for scenes comprised of several objects, which we call \emph{compositionality} and \emph{irreducibility}. Under this generative process, we prove that the ground-truth object representations can be 
identified by an invertible and compositional inference model, even in the presence of dependencies between objects. We empirically validate our results
through experiments on synthetic data. Finally, we provide evidence that our theory holds predictive power for existing object-centric models by showing a close correspondence between models’ compositionality and invertibility and their empirical identifiability.\footnote{Code/Website: \href{https://brendel-group.github.io/objects-identifiability}{brendel-group.github.io/objects-identifiability}}
\end{abstract}

\begin{figure*}[t]
    \centering
    \includegraphics[width=.9\textwidth]{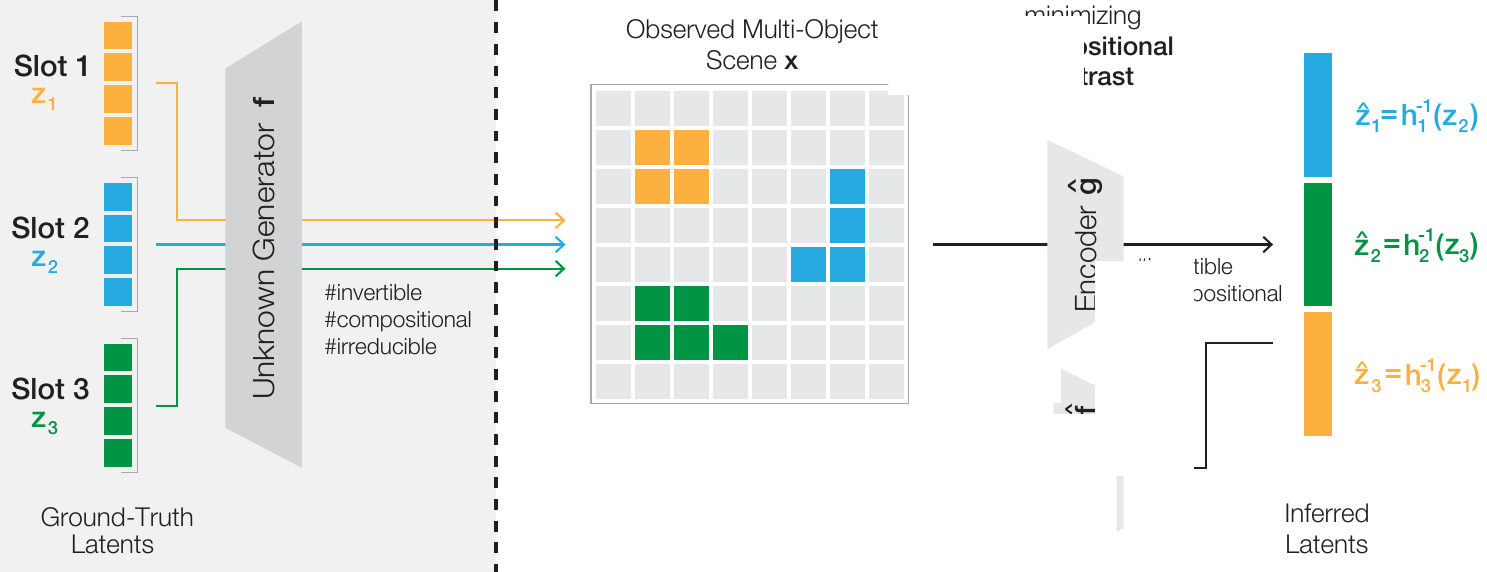}
    \caption{
    \textbf{When can unsupervised object-centric representations  provably be learned?} We assume that observed scenes~$\xb$ 
     comprising $K$ objects
    are rendered by an unknown generator~$\fb$ 
    from multiple ground-truth latent slots $\zb_1,..., \zb_K$ (here, $K=3$). We assume that this generative model has two key properties, which we call \textit{compositionality} (\cref{def:compositional}) and \textit{irreducibility} (\cref{def:irreducible_mechanism}). Under this model, we prove~(\cref{thm:slot_identifiability}): An invertible inference model with a compositional inverse
    yields 
    latent slots $\zbh_i$ which identify the ground-truth slots up to permutation and slot-wise invertible functions $\hb_i$ (\textit{slot identifiability}, \cref{def:slot_identifiability}). 
    To measure violations of compositionality in practice, we introduce a contrast function~(\cref{def:compositional_contrast}) which is zero if and only if a function is compositional, while to measure invertibility, we rely on the reconstruction loss in an auto-encoder framework.
    }
    \label{fig:header}
\end{figure*}

\section{Introduction}
Human intelligence exhibits an unparalleled ability to 
generalize
from a limited amount of experience 
to a wide range of novel situations~\citep{Tenenbaum2011HowTG}. 
To build machines with similar capabilities, a fundamental question is what types of abstract representations of sensory inputs 
enable such generalization~\citep{goyal2020inductive}. Research in cognitive psychology suggests that one key abstraction is the ability to represent visual scenes in terms of 
individual objects~\citep{Spelke2003WhatMU, spelke2007core, dehaene2020we,peters2021capturing}. Such \textit{object-centric representations} are thought to facilitate core cognitive abilities such as compositional generalization~\citep{fodor1988connectionism,lake2017building,Battaglia2018RelationalIB,greff2020binding} and causal reasoning over discrete concepts~\citep{Marcus2001-MARTAM-10,gopnik2004theory,gerstenberg2017intuitive,gerstenberg2021counterfactual}.

Significant effort has thus gone into endowing machine learning models with the capacity to learn object-centric representations from raw visual input. While initial approaches 
were mostly supervised ~\citep{Ronneberger2015UNetCN,he2017mask,chen2017deeplab}, a recent wave of new methods explore learning object-centric representations without direct supervision~\citep{greff2019multi,burgess2019monet,lin2020space,kipf2019contrastive,locatello2020object,weis2021benchmarking,biza2023invariant}. These methods have begun exhibiting impressive results, showing potential to scale to complex visual scenes~\citep{caron2021emerging,singh2021illiterate,Sajjadi2022ObjectSR,Seitzer2022BridgingTG} and real-world video datasets~\citep{kipf2021conditional,Singh2022SimpleUO,Elsayed2022SAViTE}.
 
Yet, despite this empirical progress, we still lack a \textit{theoretical} understanding of when unsupervised object-centric representation learning is possible. This makes it challenging to isolate the reasons underlying the success and failure of existing object-centric models and to develop principled ways to improve them. Furthermore, it is currently not possible to design novel object-centric methods that are theoretically grounded and not solely based on heuristics, many of which break down in more realistic settings~\citep{karazija2021clevrtex,Papa2022InductiveBF,Yang2022PromisingOE}. 

In the present work, we aim to address this deficiency by investigating when object-centric representations can \textit{provably} be learned without any supervision. To this end, we first specify a data-generating process for multi-object scenes as a structured latent variable model in which each object is described by a subset of latents, or a latent \textit{slot}. We then study the \emph{identifiability}
of object-centric representations under this model, i.e., we investigate under which conditions an inference model will be guaranteed to recover the subset of ground-truth latents for each object. 

\looseness-1 Because identifying the ground-truth latent variables is impossible without further assumptions on the generative process~\citep{HYVARINEN1999429,locatello2019challenging}, previous identifiability results primarily rely on distributional assumptions on the 
latents~\citep{hyvarinen2016unsupervised,Hyvrinen2017NonlinearIO,hyvarinen2019nonlinear,khemakhem2020variational,khemakhem2020ice,klindt2020towards,zimmermann2021contrastive}. In contrast, we make no such assumptions,
thus allowing for 
arbitrary statistical and causal dependencies between objects.

\paragraph{Structure and Main Contributions.}
 In the present work, we instead take the position
 that the object-centric nature of the problem imposes a very specific \textit{structure} on the \textit{generator function} that renders scenes from latent slots~(\cref{sec:model}).
Specifically, we define two key properties that this function should satisfy: \textit{compositionality}~(\cref{def:compositional}) and  \textit{irreducibility}~(\cref{def:irreducible_mechanism}). Informally, these properties imply that every pixel can only correspond to one object and that information is shared across different parts of the same object 
but not between parts of different objects---inspired by the principle of independent causal mechanisms~\citep{peters2017elements}.
Under this generative model, we then prove in~\cref{sec:theory} our \textit{main theoretical result}: the ground-truth latent slots can be identified without supervision by an invertible inference model with a compositional inverse~(\cref{thm:slot_identifiability}). To quantify compositionality, we introduce a \textit{contrast function}~(\cref{def:compositional_contrast}) that is zero if and only if a function is compositional; to quantify invertibility, we rely on reconstruction error.
We validate on synthetic data that inference models which
maximize invertibility and compositionality
indeed identify the ground-truth latent slots, even with dependencies between latents~(\cref{subsec:exp1}).
 Finally, we examine existing object-centric learning models on image data and find a close correspondence between models’ compositionality and invertibility and their success
in identifying the ground-truth latent slots~(\cref{subsec:exp2}).

To the best of our knowledge, the present work provides the first identifiability result for object-centric representations. We hope that this lays the groundwork for a better understanding of success and failure in unsupervised object-centric learning, and that future work can build on these insights to develop more effective learning methods.

\paragraph{Notation.} Bold lowercase $\zb$ denotes  vectors, bold uppercase $\Jb$ denotes matrices.
For $n\in\mathbb{N}$, let $[n]$ denote the set 
$\{\,1,\dots, n\,\}$. Additionally, if $\fb$ is a function with $n$ component functions, let $\fb_S$ denote the restriction of $\fb$ to the component functions indexed by $S\subseteq[n]$, i.e., $\fb_{S} :=(f_s)_{s\in S}$.

\begin{figure*}[t]
\centering
  \includegraphics[width=.9\textwidth]{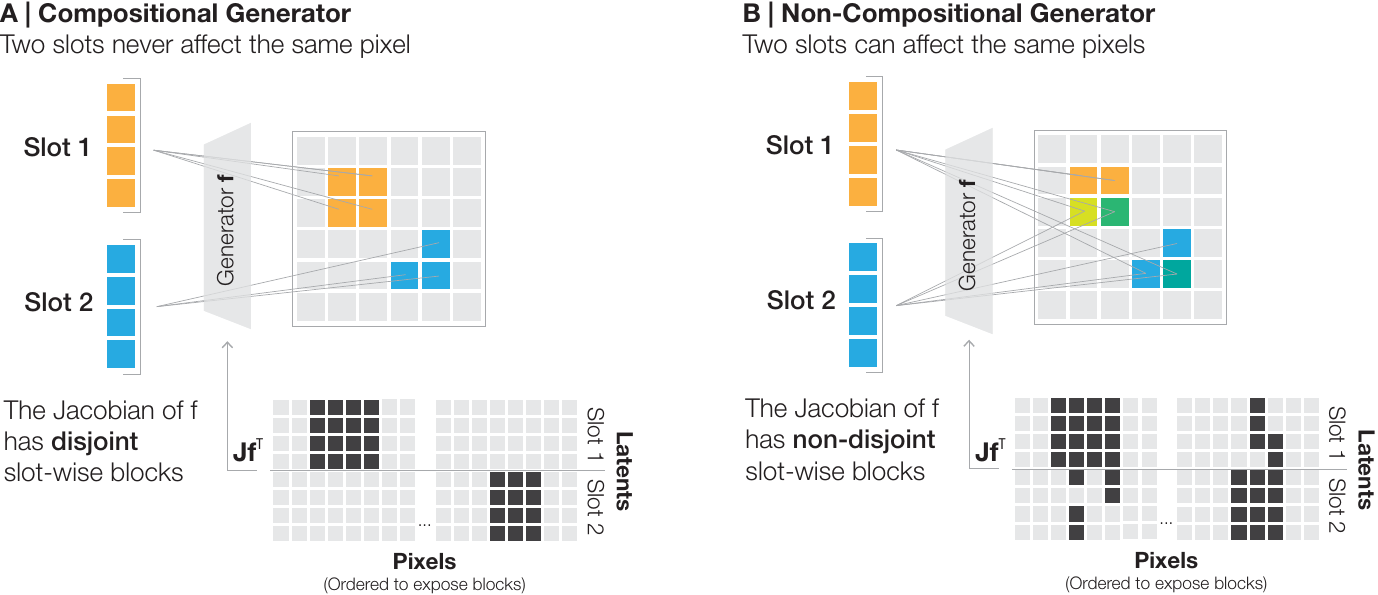}
  \caption{\textbf{Difference between a compositional and a non-compositional generator.}
  \textbf{(A)} For a compositional generator $\fb$, every pixel is affected by at most one latent slot. As a result, there always exists an ordering of the pixels such that the generator's Jacobian $\Jb\fb$ consists of disjoint blocks, one for each latent slot \textit{(bottom)}. %
  Note that both the pixel ordering and the specific structure of the Jacobian are not fixed across scenes and might depend on the latent input~$\zb$. 
  \textbf{(B)} For a non-compositional generator, there exists no pixel ordering that exposes such a structure in the Jacobian, since the same pixel can be affected by more than one latent slot.}
  \label{fig:compositional_function_visualisation}
\end{figure*}

\section{Generative Model}
\label{sec:model}
While humans have a clear intuition for what constitutes an object, formalizing this notion mathematically is not straightforward.
Indeed, there is no universally agreed-upon definition of an object; %
various formalizations based upon distinct criteria co-exist~\citep{green2019perceptualobj,Spelke1990PrinciplesOO,Koffka1936PrinciplesOG,greff2020binding}. 
We approach the problem by defining multi-object scenes in terms of a latent variable model (see~\cref{fig:header} for an overview) and argue that the object-centric nature of the problem necessitates a very specific structure on the generator, which we leverage in~\cref{sec:theory} to prove our identifiability result.

As a starting point, we assume that observed data samples~$\xx$
of multi-object scenes 
are generated from a set of latent random vectors $\zz$
through a 
diffeomorphism\footnote{a differentiable bijection with differentiable inverse}  
$\fb: \Z\rightarrow \X$, mapping from a \emph{latent} space $\Z$ to an \emph{observation} space $\X$,
\begin{equation}\label{eqn:10.5}
    \zz \sim p_\zz,\qquad\qquad \xx=\fb(\zz).
\end{equation}
The only assumption we place on $p_\zb$ is that it is fully supported on $\Z$. In particular, we do not require independence and allow for arbitrary dependencies between components of $\zb$, motivated by the fact that the presence or properties of certain objects may be correlated with those of other objects.

\subsection{Slots and Compositionality}
\label{subsec:slots_compositionality}
We think of an object in a scene as being encoded not by a single latent component $z_i$ but instead by a group of latents $\zz_k$ which specify its properties. For a scene comprised of $K$ objects, we thus assume that the latent space $\mathcal{Z}$ factorizes into $K$ subspaces~$\mathcal{Z}_{k}$, which we refer to as \emph{slots}. Each slot is assumed to have dimension $M$, representing, e.g., $M$~distinct object properties. More precisely, $\forall k \in [K]: \mathcal{Z}_{k}=\mathbb{R}^{M}$, and $\Z=\Z_{1}\times \dots \times\Z_{K}=\mathbb{R}^{KM}$. 

Let $\zz_k$ be the latent vector of the $k$\textsuperscript{th} slot. The full $KM$-dimensional latent \emph{scene representation} vector $\zz$  is then given by the concatenation of the latents from all slots, 
\begin{equation}
    \zz = \left(\zz_1, \dots, \zz_K\right).
\end{equation}
We would like to ensure that each latent slot $\zz_k$ is responsible for encoding a distinct object in a scene. To this end, the latent scene representation $\zz$ should be rendered by $\fb$ such that each slot generates \textit{exactly} one object~(see~\cref{fig:header}).
If $\fb$ is an arbitrary function with no additional constraints, however, this will generally not be the case.

First, $\fb$ lacks any structure which ensures that an object is not generated by more than one latent slot.
To see this, let $I_k(\zz)\subseteq [N]$ denote the subset of pixels in an image generated from scene representation $\zb$
that functionally depend on slot~$k$,
\begin{equation}\label{eqn:compositional_index_sets}
    I_k(\zb) := \left\{\,n\in [N] : \frac{\partial f_n}{\partial \zb_{k}}(\zb) \neq {\bf 0} \,\right\}.
\end{equation}
Note the dependence on $\zb$, which encodes that an object may appear in different places across different scenes.

\begin{figure*}[t]
\centering
  \includegraphics[width=.9\textwidth]{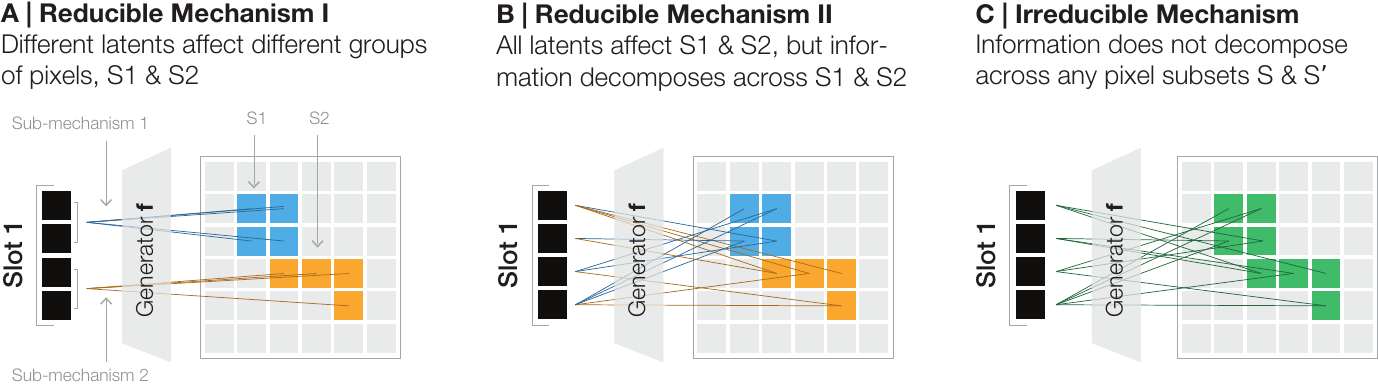}
  \vspace{-.5em}
 \caption{\textbf{(Ir)reducible mechanisms.} \textbf{(A)} A simple example of a \textit{reducible mechanism} is one for which disjoint subsets of latents from the same slot render pixel groups $S_1$ and $S_2$ separately  such that they form \textit{independent sub-mechanisms} according to~\cref{def:independent_dependent_submechanism}. This independence between sub-mechanisms is indicated by the difference in colors.  \textbf{(B)} Not all reducible mechanisms look as simple as panel A: here, $S_1$ and $S_2$ depend on every latent component in the slot, but the information in $S_1\cup S_2$ still decomposes across $S_1$ and $S_2$ as sub-mechanisms 1 and 2 are independent.
 \textbf{(C)} In contrast, for an irreducible mechanism, the information does not decompose across any pixel partition $S,  
 S^{'}$, and so it is impossible to separate it into independent sub-mechanisms.}
\label{fig:irreducibility}
\end{figure*}

Without further constraints on $\fb$, the pixel subsets $I_k(\zz)$ and $I_{j}(\zz)$ can overlap for any $k\neq j$ such that latent slots $k,j$ can affect the same pixels and thus contribute to generating the same object (see~\cref{fig:compositional_function_visualisation}B, top). To avoid this, we impose the following structure  on $\fb$, which we call {\it compositionality}.
\begin{definition}[Compositionality]\label{def:compositional}
    Let $\fb: \Z\rightarrow \X$ be differentiable. $\fb$ is said to be \emph{compositional}~if
    \begin{equation} \label{eqn:compositional}
        \forall \zb\in \Z: \qquad   k\neq j \implies   I_k(\zb) \cap I_j(\zb) = \emptyset.
    \end{equation}
\end{definition}Compositionality implies that each pixel is a function of at most one latent slot and thus imposes a local sparsity structure on the \textit{Jacobian} matrix $\Jb\fb=\big(\frac{\partial  f_i}{\partial 
 z_j}\big)_{ij}$ of $\fb$, which is visualized in \cref{fig:compositional_function_visualisation}, bottom.
Intuitively, the Jacobian of a compositional generator can always be brought into block structure through an appropriate permutation of the pixels. However, this block structure is local in that the required permutation may differ across scene representations $\zb$.

\subsection{Mechanisms and Irreducibility}
While compositionality ensures that different latent slots do not generate the same object, we need an additional constraint on $\fb$ to ensure that each slot generates only one object, rather than something humans would regard as multiple objects. 
To see this, consider the example depicted in~\cref{fig:irreducibility}A, where $\fb$ maps the first half of the latent slot to the pixels denoted~$S_{1}$ and the second half to~$S_{2}$. It is clear that for humans, these groups of pixels would likely be considered  as distinct objects. On the other hand, it is not immediately clear what formal criteria would give rise to such a distinction. 

Intuitively, the issue with the two ``sub-objects'' $S_1$ and $S_2$ in~\cref{fig:irreducibility}A appears to be that they are \textit{independent} of each other in some sense. To avoid such splitting of objects within slots, we would thus like to enforce that pixels belonging to the same object are \textit{dependent} on one another. But what is a meaningful notion of such \textit{instance-level} independence of objects? Since we are dealing with a single scene sampled according to~\cref{eqn:10.5}, it cannot be statistical in nature. Instead, our intuition is more aligned with the notion of \textit{algorithmic independence} of objects~\citep{janzing2010causal}, a formalization\footnote{albeit an impractical one formulated in terms of Kolmogorov complexity (algorithmic information), which is not computable} of the principle of independent causal mechanisms (ICM)
which posits that physical generative processes consist of ``autonomous modules that do not inform or influence each other''~\citep{peters2017elements}.
The two subsets of pixels $S_1$ and $S_2$ in~\cref{fig:irreducibility}A are independent of each other in precisely this sense: they arise from autonomous processes that do not share information.

In the following, we therefore draw inspiration from prior implementations of the ICM principle~\citep[][see~\cref{sec:related_work} for more details]{daniusis2010inferring,janzing2012information,gresele2021independent}
to formalize our intuitions about independence of objects. 
First, we define the mapping which locally renders information from the $k$\textsuperscript{th} latent slot to the affected pixels $I_k(\zz)$ which we refer to as a \emph{mechanism}.%
\begin{definition}[Mechanism]\label{def:mechanism}
    $\forall \zb\in \Z, k\in [K]$, we define the $k$\textsuperscript{th} \emph{mechanism} of $\fb$ at $\zb$ as the Jacobian matrix $\Jb\fb_{I_k}(\zb)$.
\end{definition}
The $k$\textsuperscript{th} mechanism can be understood as the sub-matrix of the Jacobian of $\fb$ whose rows correspond to the pixels $I_k(\zz)$ affected by slot $k$. Further, we define a \emph{sub-mechanism} as the restriction
to a {\it subset} of the affected pixels.
\begin{definition}[Sub-Mechanism]\label{def:submechanism}
    $\Jb\fb_{S}(\zb)$ is said to be a \emph{sub-mechanism} of $\Jb\fb_{I_k}(\zb)$, if $S\subseteq I_k(\zb)$ and $S$ is nonempty.
\end{definition}
\looseness-1 In light of these definitions, \cref{fig:irreducibility}A 
consists of two sub-mechanism, $\Jb\fb_{S_1}(\zb)$ and $\Jb\fb_{S_2}(\zb)$, which generate pixels $S_{1}$ and $S_{2}$. To characterize the level of dependence between sets of pixels and their associated sub-mechanisms,
we propose to use the matrix \textit{rank}, which can be seen as a non-statistical measure of information as it
locally characterizes the latent capacity used to generate the corresponding pixels.
\begin{definition}[Independent/Dependent Sub-Mechanisms]\label{def:independent_dependent_submechanism}
    Let $S_1, S_2 \subseteq [N]$ and $\zb\in\Z$. The sub-mechanisms $\Jb\fb_{S_1}(\zb)$ and $\Jb\fb_{S_2}(\zb)$ are said to be \emph{independent} if:
    \begin{equation} \label{eq:independent_dependent_submechanism_independent}
        \rank\left(\Jb\fb_{S_1\cup S_2}(\zb)\right)=\rank\left(\Jb\fb_{S_1}(\zb)\right)+\rank\left(\Jb\fb_{S_2}(\zb)\right).
    \end{equation}
    Conversely, they are said to be \emph{dependent} if:
    \begin{equation*}\label{eqn:dependent_partition}
        \rank \left( \Jb\fb_{S_1\cup S_2}(\zb)\right) < \rank \left( \Jb\fb_{S_1}(\zb)\right) + \rank \left( \Jb\fb_{S_2}(\zb)\right).
    \end{equation*}
\end{definition}
Intuitively, two sub-mechanisms $\Jb\fb_{S_1}(\zb)$ and $\Jb\fb_{S_2}(\zb)$ are independent according to~\cref{def:independent_dependent_submechanism} if 
the information content of pixels $S_1\cup S_2$ decomposes across $S_1$ and $S_2$ in the sense that the latent capacity required to \textit{jointly} generate $S_1\cup S_2$ (LHS of~\cref{eq:independent_dependent_submechanism_independent}) is the same as that required to generate $S_1$ and~$S_2$ \textit{separately} (RHS of~\cref{eq:independent_dependent_submechanism_independent}). Such a decomposition will occur when the rows of the sub-mechanism $\Jb\fb_{S_1}(\zb)$ do not lie in the row-space of the sub-mechanism $\Jb\fb_{S_2}(\zb)$ and vice-versa. This will be the case in~\cref{fig:irreducibility}A where $\Jb\fb_{S_1}(\zb)$ and $\Jb\fb_{S_2}(\zb)$ affect different pixels since the rows of the Jacobian for pixels $S_1$ and $S_2$ will never have non-zero entries for the same column. As shown in~\cref{fig:irreducibility}B, however, it could also be the case that all latents within a slot affect pixels in both $S_1$ and $S_2$, yet the information content of $S_1\cup S_2$ still decomposes across $S_1$ and $S_2$ since the rows of $\Jb\fb_{S_1}(\zb)$ and $\Jb\fb_{S_2}(\zb)$ could span linearly independent subspaces.

To enforce that each slot generates only one object, we now finally place the condition on the mechanisms of $\fb$ that they cannot be partitioned into independent sub-mechanisms (see~\cref{fig:irreducibility}C). We refer to 
this property as {\it irreducibility}.
\begin{definition}[Irreducibility] \label{def:irreducible_mechanism}
    $\fb$ is said to have \emph{irreducible mechanisms}, or is \emph{irreducible}, if for all
    $\zb\in \Z$, $k\in [K]$ and any partition of $I_k(\zb)$ into $S_1$ and $ S_2$,  the sub-mechanisms $\Jb\fb_{S_1}(\zb)$ and $\Jb\fb_{S_2}(\zb)$ are dependent in the sense of \cref{def:independent_dependent_submechanism}.
\end{definition}

\section{Theory: Slot Identifiability }
\label{sec:theory}
Given multi-object scenes sampled from the generative model outlined in \cref{sec:model}, we now seek to understand under what conditions an
{\it inference model} $\hat{\gb}:\X\to\Z$ will provably identify
the ground-truth object representations.
Ideally, we would like $\hat{\gb}$ to recover the true inverse $\gb:=\fb^{-1}$, but that is generally only possible up to certain irresolvable ambiguities. 
In our multi-object setting, the objective is to separate the object representations such that each inferred slot captures 
{\it one and only one} ground-truth slot.
We refer to this notion as {\it slot identifiability} 
and define it as follows. 
\begin{definition}[Slot Identifiability]
\label{def:slot_identifiability}
    Let $\fb:\Z\to\X$ be a diffeomorphism.
    An inference model $\hat{\gb}: \X\rightarrow \Z$ is said to \emph{slot-identify} $\zb = \gb(\xb)$ via $\hat{\zb}=\hat{\gb}(\xb)=\hat{\gb}(\fb(\zb))$ if for all $k\in [K]$ there exist a unique
    $j\in [K]$ and a diffeomorphism
    $\hb_k: \Z_{k}\rightarrow \Z_{j}$ such that $\hat \zb_j = \hb_k(\zb_k)$ for all $\zb\in\Z$.
\end{definition}
We are now in a position to state our main theoretical result (all complete proofs are provided in~\cref{sec:proofs}).
\begin{restatable}[]{theorem}{identifiability} \label{thm:slot_identifiability}
    Let 
    $\fb : \Z \to \X$ be a diffeomorphism that is compositional (\cref{def:compositional}) with irreducible mechanisms (\cref{def:irreducible_mechanism}). 
    If an inference model $\hat\gb: \X\rightarrow \Z$ is \emph{(i)} a diffeomorphism with \emph{(ii)} compositional inverse $\hat{\fb}=\hat{\gb}^{-1}$, then $\hat{\gb}$
    slot-identifies $\zb = \gb(\xb)$ in the sense of~\cref{def:slot_identifiability}.
\end{restatable}
\paragraph{Proof Sketch.}
Irreducibility of $\fb$ ensures that information is shared across different parts of an object, and compositionality of $\fb$ that this information is not shared with other objects. This creates an asymmetry in the latent capacity required to encode the entirety of one object compared to parts of different objects. When $\hat\gb$ satisfies (i) and (ii), this asymmetry can be leveraged to show that each inferred slot $\hat\zb_{j}$ maps to \textit{one} and \textit{only} ground-truth slot $\zb_{k}$ by a \textit{proof by contradiction}. Namely, suppose that $\hat\gb$ maps pixels of two distinct objects to the same slot $j$. If $\hat\gb$ were to encode all latent information required to generate these pixels in slot~$j$, there would not be sufficient total latent capacity to recover the entire scene, leading to a violation of (i) invertibility. Hence, information for at least one of the pixels needs to be distributed across multiple slots, violating (ii) compositionality of $\hat\fb=\hat\gb^{-1}$.

\paragraph{Implications for Object-Centric Learning.}
\cref{thm:slot_identifiability} highlights important conceptual points for object-centric representation learning. First, it shows that distributional assumptions on the latents $\zb$ are not necessary for slot identifiability;
instead, it suffices to enforce structure on the generator $\fb$. This falls in line with state-of-the-art (SOTA) object-centric learning methods~\citep{locatello2020object,Singh2022SimpleUO,Seitzer2022BridgingTG,Elsayed2022SAViTE}, which are based on an auto-encoding framework, thus imposing no additional structure on $p_{\zb}$. However, while these models directly enforce invertibility through the reconstruction objective, it is less clear whether and to what extent they also enforce compositionality. Specifically, compositionality is not explicitly optimized in any object-centric methods. Yet, the success of SOTA models %
in practice suggests that it may be implicitly enforced to some extent through additional inductive biases in the model. We explore this point empirically (see~\cref{fig:image_models_training_curves}) and leave a more theoretical exploration for future work.

\cref{thm:slot_identifiability} also emphasizes that using a restricted latent bottleneck plays an important role in achieving slot identifiability. Specifically, \cref{thm:slot_identifiability} is predicated on $\mathrm{dim}(\zb){=}\mathrm{dim}(\hat\zb)$ and would no longer hold in its current form if $\mathrm{dim}(\zb){<}\mathrm{dim}(\hat\zb)$. The importance of restricting the latent capacity of object-centric models was emphasized empirically by~\citet{engelcke2020reconstruction}. Yet, the most successful object-centric models in practice often use $\mathrm{dim}(\zb){<}\mathrm{dim}(\hat\zb)$~\citep{dittadi2021generalization,locatello2020object,Sajjadi2022ObjectSR}. A potential explanation for this discrepancy is that SOTA object-centric models do encode information from multiple objects in each latent slot, but this additional information is ignored by the decoder during reconstruction such that image-level segmentations remain accurate. We provide some evidence for this hypothesis through experiments with existing object-centric models in~\cref{subsec:exp2}.

\paragraph{Measuring Compositionality.}
While \cref{thm:slot_identifiability} reveals properties an inference function should satisfy to achieve slot identifiability, it presents these properties in an abstract mathematical form. If we seek to leverage \cref{thm:slot_identifiability} to assess the performance of existing object-centric models or inform new training objectives for object-centric learning, we require a way to quantify whether an inference model is (i)~a diffeomorphism and (ii) compositional.
Regarding~(i), 
one clear choice is to train an auto-encoder with differentiable encoder~$\hat\gb$ and decoder~$\hat\fb$ and minimize reconstruction loss to enforce invertibility.
Regarding (ii), on the other hand, it is much less obvious how to quantify compositionality. To this end, we introduce the following contrast function, which we prove to be zero if and only if a function is compositional:
\begin{definition}[Compositional Contrast] \label{def:compositional_contrast}
    Let $\fb: \Z\rightarrow \X$ be differentiable. The \emph{compositional contrast} of $\fb$ at $\zb$
    is
    \begin{equation} \label{eq:compositional_contrast}
        \Ccomp(\fb, \zb) = \sum\limits_{n=1}^{N} 
        \sum\limits_{\substack{k=1
        }}^{K}
        \sum\limits_{j=k+1}^{K}
        \left\|\frac{\partial f_{n}}{\partial \zb_{k}}(\zb)\right\|
        \left\|\frac{\partial f_{n}}{\partial \zb_{j}}(\zb)\right\|
        \,.
    \end{equation}%
\end{definition}%
For a given scene representation $\zb$ and generator $\fb$, the contrast function in~\cref{eq:compositional_contrast} computes the sum over all pixels~$n$ of all pairwise products of the (L2) norms of those pixels' gradients with respect to any two distinct slots $k\neq j$.
As such, it is a non-negative quantity that can only be zero if every pixel is affected by at most one slot (i.e., $\fb$ is \textit{compositional}), for otherwise there would be a pair of slots $k\neq j$ for which the gradient norms are both non-zero resulting in their product being non-zero. 

We leverage this characterization of compositionality to provide our second result, which can be viewed as an optimization-based perspective on~\cref{thm:slot_identifiability}.
\begin{restatable}[]{theorem}{contrast} \label{thm:compositional_contrast_identifiable}
    Let 
    $\fb : \Z \to \X$ be a diffeomorphism that is compositional (\cref{def:compositional}) with irreducible mechanisms (\cref{def:irreducible_mechanism}).
    If an encoder $\hat\gb: \X\rightarrow \Z$ and decoder~$\hat\fb: \Z\rightarrow \X$ are both differentiable and solve the following functional equation
    \begin{equation}
    \label{eq:objective}
        \mathbb{E}_{\xx \sim p_\xx}\left[\left\|\hat \fb(\hat\gb(\xb))-\xb\right\|^{2}_{2} + \lambda \Ccomp\left(\hat\fb, \hat\gb(\xb)\right)\right] = 0,
    \end{equation}
    for $\lambda >0$, then $\hat\gb$ slot-identifies $\zb$ in the sense of \cref{def:slot_identifiability}.
\end{restatable}

\section{Related Work} \label{sec:related_work}
\paragraph{Object-Centric Generative Models.}
Prior works have also formulated generative models for multi-object scenes based on latent slots~\citep{LeRoux2011LearningAG,Heess2012LearningGM,greff2015binding,greff2017neural,greff2019multi,van2018relational,von2020towards,engelcke2019genesis,Engelcke2021GENESISV2IU}, though without studying identifiability. 
Our assumptions on the
generative model~(\cref{sec:model}) bear intuitive similarity to some of these prior works, but they also differ in several fundamental ways.
First, 
compositionality~(\cref{def:compositional})
is stated as a desideratum for nearly all object-centric generative models.
Yet, this constraint is not actually enforced by most existing approaches, particularly those 
based on spatial mixture models in which every slot may affect every pixel~\citep{greff2015binding,greff2017neural,greff2019multi,van2018relational,engelcke2019genesis,Engelcke2021GENESISV2IU}.
More closely related is a dead-leaves model approach, in which a scene is sequentially generated by layering objects such that each pixel is affected by at most one slot~\citep{LeRoux2011LearningAG, von2020towards, tangemann2021unsupervised}. In contrast, we define compositionality directly through assumptions on the structure of the (Jacobian of the) generator.
Second, 
our irreducibility criterion~(\cref{def:independent_dependent_submechanism,def:irreducible_mechanism}) bears conceptual similarity to prior works, which assume that different objects do not share information whereas parts of the same object do~\citep{Hyvrinen2006LearningTS,greff2015binding,greff2017neural,van2018relational}.
Importantly, however, these works formalize this intuition using statistical criteria such as \textit{statistical independence} between pixels from different objects and dependence between pixels from the same object.
However, this leads to an incorrect characterization of objects: e.g., the presence of a coffee cup should increase the likelihood that a table is also present, despite these being separate objects~\citep{trauble2021disentangled,scholkopf2021toward}.
Here, we instead formulate independence/dependence between objects in a \textit{non-statistical} sense, inspired by algorithmic independence of mechanisms.

\paragraph{Objects and Causal Mechanisms.}
In causal modelling~\citep{spirtes2001causation,pearl_2009}, a \textit{mechanism} typically refers to a function that determines the value of an effect variable from its direct causes and possibly a noise term, leading to a conditional distribution of effect given causes.
Thus, we could view objects as the effects of the latent variables that cause them.
While the causal variables are generally not independent, it has been argued that the mechanisms producing them should be~\citep{scholkopf2012causal,peters2017elements}. Since this is an independence between functions or conditionals rather than between random variables, it is non-trivial to formalize it statistically~\citep{janzing2010causal,guo2022causal}.
Hence, various implementations of the principle have been proposed~\citep{daniusis2010inferring,janzing2010telling,janzing2012information,shajarisales2015telling,locatello2018competitive,besserve2018group,besserve2021theory,Janzing+2021+285+301}, typically for settings in which both cause and effect are observed.
Our notion of independent sub-mechanisms
is most closely related to work by~\citet{gresele2021independent}, who also study representation learning and define mechanisms more broadly in terms of the Jacobian~$\Jb\fb$:
they assume independent latents and formalize mechanism independence  as column-orthogonality of the Jacobian. In contrast, our rank condition~(\cref{eq:independent_dependent_submechanism_independent}) is inspired by object-centric representation learning with dependent latents.

\paragraph{Identifiable Representation Learning.}
\looseness-1 As this is the first identifiability study of unsupervised object-centric representations, our problem setting differs from existing work both in terms of the assumptions we make on the generative process and the type of identifiability that we aim to achieve.
First, prior work on identifiable representation learning commonly places assumptions on the latent distribution, such as conditional independence given an auxiliary variable~\citep{hyvarinen2016unsupervised,Hyvrinen2017NonlinearIO,hyvarinen2019nonlinear,khemakhem2020variational,halva2020hidden,Halva2021} or access to views arising from pairs of similar latents ~\citep{gresele2020incomplete,klindt2020towards,zimmermann2021contrastive,von2021self}, while leaving the generator $\fb$ completely unconstrained. In contrast, we place no assumptions on $p_{\zb}$ and instead impose structure on (the Jacobian of) the generator $\fb$. 
Recent works have also leveraged assumptions on $\Jb\fb$ such as orthogonality~\citep{gresele2021independent,zheng2022on,Reizinger2022EmbraceTG,Buchholz2022FunctionCF}, unit determinant~\citep{yang2022nonlinear}, or a \textit{fixed} sparsity structure~\citep{moran2021identifiable,lachapelle2021disentanglement,lachapelle2022partial}.
While the latter relates to our definition of compositionality~(\cref{def:compositional}), 
we crucially allow the sparsity pattern 
on $\Jb\fb$ to vary with $\zb$ (in line with the basic notion that objects are not fixed in space), and impose sparsity with respect to slots rather than individual latents.  
Secondly, existing work typically aims to identify individual latent components $z_i$ up to permutations (or linear transformations). However, this is inappropriate for object-centric representation learning, where we aim to capture and isolate the subsets of latents corresponding to each object in well-defined slots. Identifying such groups of latents is similar to efforts in independent subspace analysis~\citep[ISA;][]{DBLP:journals/neco/HyvarinenH00}. However, results for ISA are generally restricted to linear models and independent groups, whereas we allow for nonlinear models and dependence.
Our notion of slot identifiability is closest related to that of block-identifiability introduced by~\citep{von2021self} 
and can be seen as an extension or generalization thereof to a setting with multiple blocks.

\begin{figure*}[tb]
    \centering
    \includegraphics[width=\linewidth]{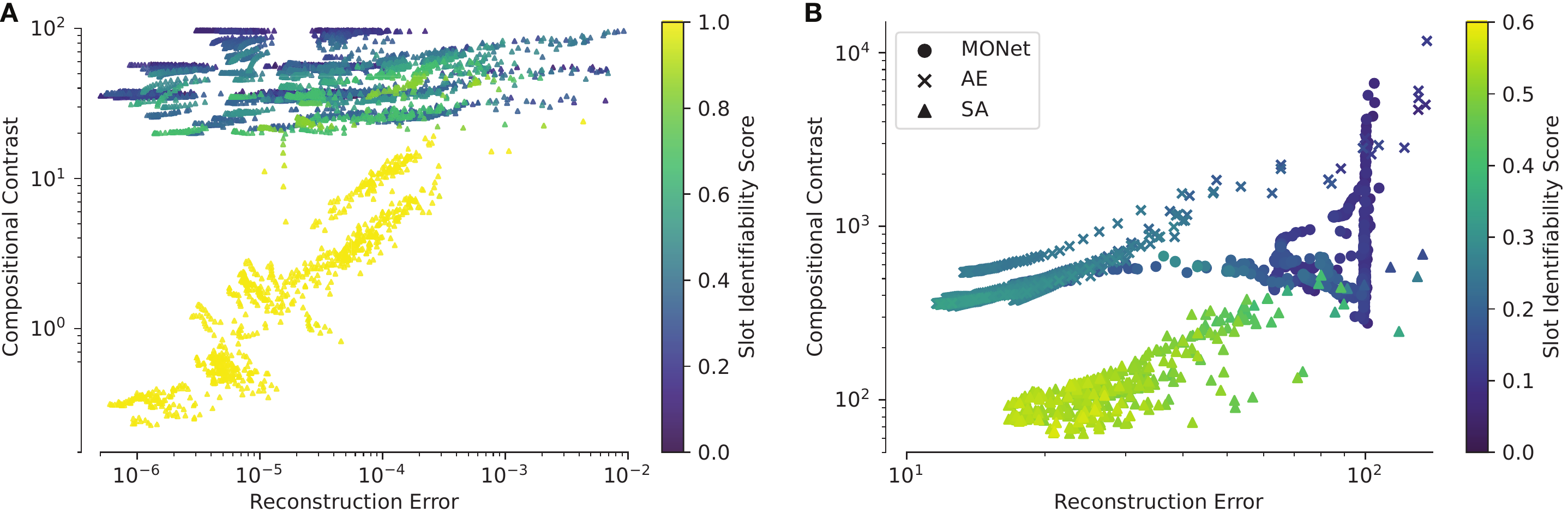}
    \vspace{-2em}
    \caption{\textbf{(A) Experimental validation of \cref{thm:compositional_contrast_identifiable}.} We trained models on synthetic data generated according to \cref{sec:model} with 2, 3, 5 independent latent slots (see~\cref{subsec:exp1}). The color coding indicates the level of identifiability achieved by the model, measured by the Slot Identifiability Score (SIS), where higher values correspond to more identifiable models. As predicted by our theory, if a model sufficiently minimizes both reconstruction error and compositional contrast, then it identifies the ground-truth latent slots.
    \textbf{(B) Application of \cref{thm:compositional_contrast_identifiable} to existing object-centric models.} We train 3 existing object-centric architectures---MONet, Slot Attention (SA), and an additive auto-encoder (AE)---on image data and visualize their SIS as a function of both reconstruction error and compositional contrast. We see across models that, in general, SIS increases as reconstruction error and compositional contrast are minimized.}
    \label{fig:results_toy_image_experiments}
\end{figure*}

\section{Experiments}
\label{sec:experiments}
\Cref{thm:compositional_contrast_identifiable} states that inference models which minimize reconstruction loss $\mathcal{L}_{\mathrm{rec}}$ and compositional contrast $\Ccomp$ achieve \emph{slot identifiability} (\cref{def:slot_identifiability}). This provides a concrete way to empirically test our main theoretical result.
To do so, we perform two main sets of experiments. First, in \cref{subsec:exp1} we generate controlled synthetic data according to the process specified in \cref{sec:model} and train an inference model on this data which directly optimizes $\mathcal{L}_{\mathrm{rec}}$ and $\Ccomp$ jointly. Second, in \cref{subsec:exp2} we seek to better understand the relationship between $\mathcal{L}_{\mathrm{rec}}$, $\Ccomp$, and slot identifiability in existing object-centric models. To this end, we analyze a set of models trained on a multi-object sprites dataset. %

\paragraph{Quantifying Slot Identifiability.}
To assess whether a model is slot identifiable in practice, we first establish a metric to measure slot identifiability. Specifically, we want to measure if there exists an invertible function between each ground-truth and exactly one inferred latent slot. To this
end, we first fit nonlinear models between inferred and ground-truth slots and measure their quality by the $R^{2}$ coefficient of determination. %
To properly measure this $R^{2}$ score, we must first match each ground-truth slot to its corresponding inferred slot as permutations could exist between slots. For our experiments in \cref{subsec:exp1}, this permutation will be global i.e. the same for all inferred latents, thus we use the Hungarian Algorithm~\citep{kuhn1955hungarian} to find the optimal matching based on the $R^{2}$ scores for models fit between every pair of slots. For our experiments on image data in \cref{subsec:exp2}, however, such a permutation will be local due to the permutation invariance of the generator. To resolve this, we follow a procedure similar to that of~\citet{locatello2020object} and \citet{dittadi2021generalization} using online matching when fitting models between slots. Specifically, at every training iteration, we compute a matching loss for each sample for all possible pairings of ground-truth and inferred slots and use the Hungarian algorithm to find the optimal assignment for minimizing this loss. After resolving permutations, the $R^{2}$ scores for the matched slots tell us how much information about each ground-truth slot is contained in one inferred slot. We also need to ensure, however, that inferred slots only contain information about one ground-truth slot and not multiple. To this end, we correct this score by subtracting the maximum $R^{2}$ score from models fit between each inferred latent slot and the ground-truth slots that it was not previously matched with. %
Taking the mean of this score across all slots yields the final score, which we refer to as the \emph{slot identifiability score} (SIS). Further details on the metric are given in~\cref{subsec:sis_details}.%

\subsection{Synthetic Data} \label{subsec:exp1}
\paragraph{Experimental Setup.}
To generate synthetic data according to \cref{sec:model}, we first sample a $KM$-dimensional latent vector from a normal distribution $p_\zb = \mathcal{N}(0, \Sigma)$, where we consider scenarios with both statistically independent latents ($\Sigma=\Ib$) and dependent latents ($\Sigma\sim\mathrm{Wishart}_{KM}(\Ib, KM)$). We then partition the latent vector into $K$ slots, each with dimension $M$, and apply the same multi-layer perceptron (MLP) to each of the $K$ slots separately. The MLP has $2$ layers, uses LeakyReLU non-linearities, and is chosen to lead to invertibility almost surely by following the settings used in previous work~\citep{hyvarinen2016unsupervised,Hyvrinen2017NonlinearIO,zimmermann2021contrastive}. Observations $\xb$ are obtained by concatenating the slot-wise MLP outputs such that the generator is compositional according 
 to~\cref{def:compositional} as well as invertible.\footnote{Regarding enforcing irreducibility, see~\cref{subsec:exp1_details}.}
We train models with a number of slots $K\in\{2,3,5\}$ and $\lambda\in\{10^{-7},10^{-5},10^{-2},0,1,10\}$ (see~\cref{thm:compositional_contrast_identifiable}) each across 10 random seeds (180 models in total). In all cases, we use slot-dimension $M=3$ and slot-output dimension of 20 such that $\mathrm{dim}(\xb)
= K\cdot20$. Further details on this setup may be found in~\cref{subsec:exp1_details}.

\paragraph{Results.}
In \cref{fig:results_toy_image_experiments}A, we visualize the SIS as a function of the reconstruction error and compositional contrast for independent latents for all $K\in\{2,3,5\}$. We normalize $\mathcal{L}_{\mathrm{rec}}$ and $\Ccomp$ to ensure that their scores are comparable across different $K$, which we discuss in further detail in \cref{subsec:c_comp_ext}. As predicted by \cref{thm:compositional_contrast_identifiable}, we can see that all models that minimize both objectives jointly yield high SIS, whereas models that fail to minimize, e.g., the compositional contrast achieve subpar identifiability. Results for dependent latents yield a similar trend which can be seen in \cref{fig:toy_dependent_latents}.%

\subsection{Existing Object-Centric Models} \label{subsec:exp2}
\paragraph{Experimental Setup.}
We now aim to understand the predictions made by our theory in the context of existing object-centric models trained on image data. To this end, we consider image data generated by the Spriteworld renderer \citep{spriteworld19}. Specifically, we generate images with $2$ to $4$ objects, each described by $4$ continuous (size, color, x/y position) and $1$ discrete (shape) independent latent factors. Samples of this dataset are shown in~\cref{fig:sprites_sample}. We investigate three object-centric approaches on this data: Slot Attention \citep{locatello2020object}, MONet \citep{burgess2019monet}, and an additive auto-encoder. We train all models with $4$ latent slots, each with dimension $16$, leading to an inferred latent dimension larger than the ground-truth. This discrepancy between inferred and ground-truth latent dimensionality is ubiquitous in existing object-centric models. However, it violates our theoretical assumptions which require equal dimensions. See~\cref{subsec:exp2_details} for further experimental details.

\paragraph{Results.}
SIS as a function of reconstruction error and compositional contrast is shown in \cref{fig:results_toy_image_experiments}B. Similar to \cref{fig:results_toy_image_experiments}A,
SIS tends to increase as $\mathcal{L}_{\mathrm{rec}}$ and $\Ccomp$ are minimized, highlighting that our theory holds predictive power for slot identifiability in existing object-centric models. Notably, this is in spite of our theoretical assumptions not being exactly met due to the inferred latent dimension exceeding the ground-truth. This mismatch in dimension does seem to have an effect on SIS, however, which can be seen in~\cref{fig:joint_results_spriteworld_on_off_and_on_vs_off_and_on_and_off}. Here, we can see that the  subtracted $R^{2}$ score in the SIS computation is non-zero across models suggesting that these models are using their additional latent capacity to encode information from multiple objects, despite the decoder presumably not using this information during reconstruction.

\section{Discussion}
\label{sec:discussion}

\paragraph{Limitations of Experiments.}
We emphasize that the main goal of this work is to create a theoretical foundation for object-centric learning. Hence, we focus our experiments on validating \cref{thm:compositional_contrast_identifiable} (\cref{subsec:exp1}) and
exploring our theoretical predictions in existing object-centric models (\cref{subsec:exp2}). While our experiments in \cref{subsec:exp2} provide evidence that existing models which minimize $\mathcal{L}_{\mathrm{rec}}$ and $\Ccomp$ achieve higher SIS, scaling up these experiments to more models and datasets would lead to a more comprehensive understanding of the exact extent to which the performance of existing models can be understood from our theory. We leave such a larger empirical study for future work.

\paragraph{Limitations of Theory.}
While we believe that our theoretical assumptions capture the essence of important concepts in object-centric learning, they will be violated to various degrees in practical scenarios.  For example, the assumption of compositionality~(\cref{def:compositional}) on the generator~$\fb$ is broken by translucency/reflection, as  a single pixel can then be affected by multiple latent slots. Additionally, occlusions are not yet fully covered by our theory, as pixels at the border of occluding objects would be affected by multiple latent slots. Additionally, it is common to assume in practice that the generator $\fb$ is invariant to permutations of the latent slots it acts on. %
This permutation invariance leads to a lack of invertibility of $\fb$, however, as permuted latents will give rise to the same observation. We anticipate that our theoretical results can be adapted to incorporate such a permutation invariant generator but leave this for future work.

\paragraph{Relationship to Existing Definitions of Objects.}
Under our framework, groups of pixels corresponding to an object have the property that the latent capacity needed to encode partitions of these pixels separately exceeds the latent capacity needed to encode the pixels as a whole (\cref{def:irreducible_mechanism}). Intuitively, this implies that there is latent information shared across different parts of an object. By considering the location of objects as one such latent information, our definition relates to the Gestalt law of common fate~\citep{Koffka1936PrinciplesOG,tangemann2021unsupervised} and the concept of a Spelke Object~\citep{Spelke1990PrinciplesOO,Chen2022UnsupervisedSI} which posit that pixels belonging to the same object move together. Furthermore, by considering color or texture as shared latent information, our definition relates to the Gestalt law of similarity~\citep{Koffka1936PrinciplesOG} that posits that items sharing visual features tend to be grouped together as a single object.

\paragraph{Extensions of Theory.}
While our theoretical results provide relatively general conditions under which object-centric representations can be identified, there are several potential ways our results could be extended.
First, we hypothesize that the reverse implication of our main result may hold as well, i.e., given the generative model in \cref{sec:theory}, compositionality and invertibility are not only sufficient but also necessary conditions for slot identifiability.
A formal proof of this conjecture would further highlight the importance of these properties.
Additionally, it would be interesting to aim to extend our theoretical approach to identifying not just objects but also abstractions such as part-whole hierarchies~\citep{Hinton2021HowTR} or individual object attributes. In this case, our notion of compositionality would need to be adjusted to account for abstractions that interact during generation. Lastly, it would be interesting to extend our results 
to leverage weakly-supervised information, such as motion,
which has been shown empirically to be helpful for object-centric learning~\citep{tangemann2021unsupervised,kipf2021conditional,Elsayed2022SAViTE,Chen2022UnsupervisedSI}.

\paragraph{Optimizing $\Ccomp$ in Object-Centric Models.}
While creating a new method for object-centric learning is not the focus of this work, one question based on \cref{thm:compositional_contrast_identifiable} is whether $\Ccomp$ can be optimized directly in object-centric models on image data 
to improve slot identifiability. In this setting, explicitly optimizing $\Ccomp$, as was done in~\cref{subsec:exp1}, is challenging as the contrast in its current form is based on Jacobians. Thus, naively optimizing it through gradient descent corresponds to second-order optimization, which creates computational challenges for larger models and data dimensionalities. As previously noted, it could also be the case that there exist implicit ways to enforce that $\Ccomp$ is minimized, which could be occurring to some extent through inductive biases in existing object-centric models. We leave finding computationally efficient ways to minimize $\Ccomp$, whether explicit or implicit, for future work.
\paragraph{Concluding Remarks.}
Representing scenes in terms of objects is a key aspect of visual intelligence and an important component of generalization in humans. While empirical object-centric learning methods are increasingly successful, we have thus far been lacking a precise theoretical understanding of what properties of the data and model are sufficient to provably learn object-centric representations. To the best of our knowledge, this work is the first to provide such a theoretical understanding. Along with invertibility, two intuitive assumptions on the generator---compositionality and irreducibility--are sufficient to identify the ground-truth object representations. By extending identifiability theory towards object-centric learning, we hope to facilitate a deeper understanding of existing object-centric models as well as provide a solid foundation for the next generation of models to build upon.
\paragraph{Author Contributions.}
JB developed the theory with technical help from RSZ, insight from YS, and advising from WB and JvK. JB implemented and executed the experiments with help from RSZ and YS, while RSZ implemented the $\Ccomp$ and SIS metrics on image data. JB and JvK led the writing of the manuscript with help from WB, BS, and RSZ. WB and RSZ created all figures in the manuscript.

\paragraph{Acknowledgments.}
We thank: the anonymous reviewers for helpful suggestions which led to improvements in the manuscript, Andrea Dittadi for helpful discussions regarding experiments, Attila Juhos for pointing out an issue with \cref{thm:slot_identifiability}, Amin Charusaie, Michel Besserve, and Simon Buchholz for helpful technical discussions, and Zac Cranko for theoretical efforts in the early stages of the project.

This work was supported by the German Federal Ministry of Education and Research (BMBF): Tübingen AI Center, FKZ: 01IS18039A, 01IS18039B. WB acknowledges financial support via an Emmy Noether Grant funded by the German Research Foundation (DFG) under grant no. BR 6382/1-1 and via the Open Philantropy Foundation funded by the Good Ventures Foundation. WB is a member of the Machine Learning Cluster of Excellence, EXC number 2064/1 – Project number 390727645.
The authors thank the International Max Planck Research School for Intelligent Systems (IMPRS-IS) for supporting RSZ and YS.

\bibliography{references}
\bibliographystyle{icml2022}

\newpage
\appendix
\onecolumn
\section{Proofs}
\label{sec:proofs}
In this section, we present the proofs for the results presented in the main text. 
First, we recall our notation:
\paragraph{Notation.} $N$ will denote the dimensionality of observations ${\bf x}$, $K$ the number of latent slots, and $M$ the dimensionality of each latent slot ${\bf z_{k}}$. For $n\in\mathbb{N}$, $[n]$ will denote the set of natural numbers from $1$ to $n$, i.e., $[n] := \{1,\dots, n\}$. If $\fb$ is a function with $n$ component functions, then $\fb_S$ will denote the restriction of $\fb$ to the component functions indexed by $S\subseteq[n]$, i.e. $\fb_{S} :=(f_s)_{s\in S}$ where $\fb_{S}$ is ordered according to the natural ordering of the elements of $S$. Additionally, when restricting $\fb$ to the component functions indexed by $I_k(\zb)$, defined according to \cref{eqn:compositional_index_sets}, we will drop the dependence on $\zb$ for notional convenience i.e. $\fb_{I_k}(\zb) := \fb_{I_k(\zb)}(\zb)$. For functions $\fb, \hat\fb$, we will use ${I_k}(\zb), \hat{I_k}(\hat\zb)$, respectively, to distinguish between the indices defined for each function according to \cref{eqn:compositional_index_sets}. Lastly, we will slightly abuse notation and use ${\bf 0}$ to denote both the zero vector and a matrix whose entries are all $0$.

We begin by proving several lemmata which will be leveraged for our main theoretical result. We start with the intuitive result that sub-mechanisms from different latent slots are independent in the sense of~\cref{def:independent_dependent_submechanism}.
\vspace{\topsep}
\begin{lemma}[Sub-Mechanisms of Distinct Mechanisms are Independent] \label{lem:independent_partition}
    Let $\fb$ be a diffeomorphism that is compositional (\cref{def:compositional}), and let $S_1,S_2 \subseteq [N]$ be nonempty.
    $\forall \zb\in \Z, k\in [K]$, if $S_1\subseteq I_k(\zb)$, $S_2 \cap I_k(\zb) =\emptyset$, then sub-mechanisms $\Jb\fb_{S_1}(\zb)$, $\Jb\fb_{S_2}(\zb)$ are independent in the sense of \cref{def:independent_dependent_submechanism}.
\end{lemma}
\begin{proof}\vspace{-1.8\topsep}
    From the definition of $I_k(\zb)$ in \cref{eqn:compositional_index_sets} it follows that:
    \begin{align*}
        \forall n \in [N]: \quad  \frac{\partial f_n}{\partial \zb_k}(\zb) \neq {\bf 0} \implies n \in I_k(\zb).
    \end{align*}
    Since $S_1\subseteq {I_k(\zb)}$, we know that $\forall n \in S_1: \frac{\partial f_n}{\partial \zb_k}(\zb) \neq {\bf 0}$. Further, since $S_2 \cap {I_k(\zb)} = \emptyset$ it means that $\forall n \in S_2: \frac{\partial f_n}{\partial \zb_k}(\zb) = {\bf 0}$. 
    Put differently, this means that rows of $\Jb\fb_{S_1}(\zb)$ are non-zero for those rows where $\Jb\fb_{S_2}(\zb)$ vanishes and vice versa.
    Therefore, one cannot represent any column of $\Jb\fb_{S_1}(\zb)$ as a linear combination of those of $\Jb\fb_{S_2}(\zb)$. Hence, \begin{align*}
        \rank \left(\Jb\fb_{S_1}(\zb)\right) + \rank \left(\Jb\fb_{S_2}(\zb)\right) = \rank\left( [\Jb\fb_{S_1}(\zb); \Jb\fb_{S_2}(\zb)]\right),
    \end{align*}
    where $[\,\cdot\,;\,\cdot\,]$ denotes vertical concatenation.
    Note that the RHS is equal to $\Jb\fb_{S_1 \cup S_2}(\zb)$ up to permutations of rows (which do not change the rank). Thus, \cref{eq:independent_dependent_submechanism_independent} holds for $S_1, S_2$ showing that $\Jb\fb_{S_1}(\zb)$, $\Jb\fb_{S_2}(\zb)$ are independent in the sense of \cref{def:independent_dependent_submechanism}.
\end{proof}

We next show that the rank of each sub-mechanism is less than or equal to the latent slot-dimension dimension, $M$.
\vspace{\topsep}
\begin{lemma} \label{lem:rank_upper_bound}
    Let $\fb: \Z\rightarrow \X$ be a  diffeomorphism that is compositional (\cref{def:compositional}). $\forall \zb\in \Z, k\in [K]$, if 
    $S \subseteq I_k(\zb)$ is non-empty:
    \begin{align}
        \rank \left( \Jb\fb_{S}(\zb)\right) \leq M.
    \end{align}
\end{lemma}
\begin{proof}\vspace{-1.8\topsep}
    Since $S \subseteq I_k(\zb)$, then by compositionality of $\fb$
    \begin{equation}\label{eqn:rank_upper_bound}
        \forall \zb\in \Z, s \in S, j \in [K] \setminus \{k\}: \quad  \frac{\partial f_s}{\partial \zb_j}(\zb) = {\bf 0}.
    \end{equation}
    Thus, $\Jb\fb_{S}(\zb)$ has at most $M$ non-zero columns (those corresponding to the non-zero partials w.r.t. $\zb_k$) which implies $\rank(\Jb\fb_S(\zb)) \leq M$.
\end{proof}

\vspace{\topsep}
We now show that the rank of each mechanism is equal to the latent slot-dimension $M$.
\begin{lemma} \label{lem:compositional_rank}
    Let $\fb: \Z\rightarrow \X$ be a diffeomorphism that is compositional (\cref{def:compositional}). Then $\forall \zb\in \Z, k\in [K]$: \[\rank(\Jb\fb_{I_k}(\zb)) = M.\]
\end{lemma}
\begin{proof}\vspace{-1.8\topsep}
    First note $\fb$ is a diffeomorphism and is thus invertible. Therefore, $\Jb\fb$ must be invertible and thus have full column-rank, i.e., $\forall \zb\in \Z: \rank(\Jb\fb(\zb)) = MK$.
    
    Next, $\forall \zb\in \Z, k \in [K]$, let $I^{C}_k := [N] \setminus I_k$ denote the complement of $I_k$ in $[N]$ such that $I^{C}_k \cap I_k = \emptyset$. Thus, 
    by~\cref{lem:independent_partition}, the corresponding sub-mechanisms are independent:
    \begin{equation} \label{eqn:compositional_rank}
        \forall \zb\in \Z, k \in [K]:\quad  \rank(\Jb\fb(\zb)) = \rank(\Jb\fb_{I_k}(\zb)) + \rank(\Jb\fb_{I^{C}_k}(\zb)) = MK.
    \end{equation}
    By compositionality of $\fb$, 
    \begin{align}
        \forall \zb\in \Z, j \in [K] \setminus \{\,k\,\}: \quad  \frac{\partial \fb_{I_k}}{\partial \zb_j}(\zb) = {\bf 0}.
    \end{align}
    Thus, $\Jb\fb_{I_k}(\zb)$ has at most $M$ non-zero columns implying that $\rank(\Jb\fb_{I_k}(\zb)) \leq M$. 
    Furthermore, by definition,
    \begin{align}
        \forall \zb\in \Z :\quad  \frac{\partial \fb_{I^C_k}}{\partial \zb_k}(\zb) = {\bf 0},
    \end{align}
    which means that $\Jb\fb_{I^C_k}(\zb)$ has at most $(K - 1)M$ non-zero columns implying $\rank(\Jb\fb_{I^{C}_k}(\zb)) \leq (K - 1)M$. Inserting this result in \cref{eqn:compositional_rank} yields
    \begin{align}
        \forall \zb\in \Z, k \in [K]:\quad  M \leq MK - \rank(\Jb\fb_{I^{C}_k}(\zb))=\rank(\Jb\fb_{I_k}(\zb)) \leq M,
    \end{align}
    which can only be true if $\rank(\Jb\fb_{I_k}(\zb)) = M$.
\end{proof}
\vspace{\topsep}

Next, we show that for ground-truth generator $\fb$ and inferred generator $\hat\fb$, the sub-mechanisms at a given point with respect to the same pixel subset $S$ will be have the same rank. 
\begin{lemma} \label{lem:jacobian_rank_equality}
    Let $\fb, \hat\fb: \Z\rightarrow \X$ be diffeomorphisms with inverses $\gb, \hat\gb: \X\rightarrow \Z$, respectively. Then $\forall \zb\in \Z, S \subseteq [N]$ s.t. $S \neq \emptyset$, $\rank(\Jb\fb_{S}(\zb)) = \rank(\Jb\hat\fb_{S}(\hat \zb))$, where $\hat \zb := \hat\gb(\fb(\zb))$.
\end{lemma}

\begin{proof}\vspace{-1.8\topsep}
    First, we introduce the function
            \begin{align*}
                \hb := \hat\gb \circ \fb: \Z\rightarrow \Z \quad \textrm{ s.t. } \quad \hat \zb := \hat\gb(\fb(\zb)) = \hb(\zb),
            \end{align*}
            We can express $\fb$ as $\fb = \hat\fb \circ \hat\gb \circ \fb = \hat\fb \circ \hb$. Thus, if $S \subseteq [N], S \neq \emptyset, \ \ \fb_{S} = \hat\fb_{S} \circ \hb$.
            Therefore,
            \begin{equation}\label{eqn:jacobian_rank_equality}
                \forall \zb \in \Z, \ \ \rank(\Jb\fb_{S}(\zb)) = \rank(\Jb\hat\fb_{S}(\hat \zb) \Jb\hb(\zb)).
            \end{equation}
            Because $\hb$ is a diffeomorphism, $\Jb\hb(\zb)$ is invertible. Thus $\rank({\bf A}\Jb\hb(\zb))=\rank({\bf A})$ for any matrix ${\bf A}$ s.t. ${\bf A}\Jb\hb(\zb)$ is defined~\citep[Section~0.4.6]{Horn2012MatrixA2}. Therefore, by \cref{eqn:jacobian_rank_equality}:
            \begin{align}
                \forall \zb \in \Z, \ \ \rank(\Jb\fb_{S}(\zb)) = \rank(\Jb\hat\fb_{S}(\hat \zb)).
            \end{align}
\end{proof}

We now prove several propositions which will be used to build our main result~(\cref{thm:slot_identifiability}). Firstly, we show that each inferred latent slot depends on at least one ground-truth slot.
\begin{proposition} \label{prop:block_identifiability_at_least_one_slot}
     Let $\Z$ be a latent space, $\X$ an observation space, and $\fb: \Z\rightarrow \X$ a diffeomorphism that is compositional (\cref{def:compositional}). Let $\hat\gb: \X\rightarrow \Z$ be a diffeomorphism and $\hat \zb := \hat\gb(\fb(\zb)), \forall \zb\in \Z$. Then, $\forall \zb\in \Z, i \in [K], \exists j \in [K] : \frac{\partial \hat \zb_j}{\partial \zb_i}(\zb) \neq {\bf 0}$.
     \end{proposition}
\begin{proof}\vspace{-1.8\topsep}
    We first define the function 
    \begin{align*}
        \hb := \hat\gb \circ  \fb: \Z\rightarrow \Z \quad \textrm{ s.t. } \quad \hat \zb := \hat \gb(\fb(\zb)) = \hb(\zb).
    \end{align*}
    As $\hat\gb$ and $\fb$ are both diffeomorphisms, $\hb$ is also a diffeomorphism.

    Note that $\forall \zb\in \Z, \Jb\hb(\zb)$ is a square matrix. Furthermore, because $\hb$ is a diffeomorphism, it follows that $\forall \zb\in \Z$, $\Jb\hb(\zb)$ is full rank. This implies $\Jb\hb(\zb)$ must have all non-zero columns, which implies  
    \begin{align*}
    \forall \zb\in \Z, i \in [K], \exists j \in [K] : \frac{\partial \hat \zb_j}{\partial \zb_i}(\zb) \neq {\bf 0}.
    \end{align*}
\end{proof}
Next, we show that each inferred latent slot generates the same pixels as at most one ground-truth slot.
\begin{proposition} \label{prop:one_intersecting_object_indices}
    Let $\Z$ be a latent space and $\X$ an observation space defined as in \cref{sec:model}. Let $\fb: \Z\rightarrow \X$ be a diffeomorphism that is compositional (\cref{def:compositional}) with irreducible mechanisms (\cref{def:irreducible_mechanism}). Let $\hat\gb: \X\rightarrow \Z$ be a diffeomorphism with inverse $\hat\fb: \Z\rightarrow \X$ that is compositional (\cref{def:compositional}). Then $\forall \zb\in \Z, j\in [K]$, there exists exactly one $i\in [K] : \hat{I_j}(\hat\zb) \cap {I_i}(\zb) \neq \emptyset$, where $\hat\zb:=\hat\gb(\fb(\zb))$
\end{proposition}
\begin{proof}\vspace{-1.8\topsep}
    Our goal is to show that $\hat\fb$ maps each inferred latent slot $\hat\zb_{j}$ to pixels generated by exactly one ground-truth latent slot $\zb_{i}$.
    \paragraph{Step 1}
    We will first show that $\hat\fb$ maps each inferred latent slot $\hat\zb_{j}$ to pixels generated by at least one ground-truth latent slot $\zb_{i}$. More precisely, we aim to show:
    \begin{equation}\label{eqn:prop_2_step_1_base}
        \forall \zb\in \Z, j \in [K], \exists i \in [K]:\hat{I_j}(\hat\zb) \cap {I_i}(\zb) \neq \emptyset.
    \end{equation}  
    Suppose for a contradiction to \cref{eqn:prop_2_step_1_base} that:
    \begin{equation}\label{eqn:prop_2_step_1_cont}
        \exists \zb^*\in \Z, j \in [K], \nexists i \in [K]:\hat{I_j}(\hat\zb^*) \cap {I_i}(\zb^*) \neq \emptyset.
    \end{equation}
    We will show that this assumption leads to a contradiction and, hence, is false. Let, $\zb^*$ denote the value for which \cref{eqn:prop_2_step_1_cont} holds. \cref{eqn:prop_2_step_1_cont} coupled with the definition of ${I_i}(\zb^*)$ in \cref{eqn:compositional_index_sets} imply that there exists pixels which depend on $\hat\zb^*$ under $\hat\fb$ but not on $\zb^*$ under $\bf\f$. More precisely,
    \begin{align}
        \exists i \in \hat{I_j}(\hat\zb^*) : \Jb\hat\fb_i(\hat\zb^*)\neq {\bf 0}, \ \ \ \nexists i \in \hat{I_j}(\hat\zb^*) : \Jb\fb_i(\zb^*)\neq {\bf 0}
    \end{align}
    This then implies that:
    \begin{align}
        \rank(\Jb\hat\fb_{\hat{I_j}}(\hat\zb^*)) \neq {\bf 0}, \ \ \ \rank(\Jb\fb_{\hat{I_j}}(\zb^*)) = {\bf 0}
    \end{align}
    which contradicts the equality of Jacobian ranks between $\fb$ and $\hat\fb$ stated in \cref{lem:jacobian_rank_equality}. Thus, our assumed contradiction in \cref{eqn:prop_2_step_1_cont} cannot hold and we conclude that \cref{eqn:prop_2_step_1_base} must hold true.

    \paragraph{Step 2} 
    We will now show that $\hat\fb$ maps each inferred latent slot $\hat\zb_{j}$ to pixels generated by at most one ground-truth latent slot $\zb_{i}$. More precisely, for $C:=\{\, P \subseteq [K] : |P| > 1 \,\}$ we aim to show:
    \begin{equation}\label{eqn:prop_2_step_2_base}
        \forall \zb\in \Z, j \in [K], \nexists P \in C :\qquad   i \in P \implies \hat{I_j}(\hat\zb) \cap {I_i}(\zb) \neq \emptyset.
    \end{equation}
    Suppose for a contradiction to \cref{eqn:prop_2_step_2_base} that:
    \begin{equation}\label{eqn:prop_2_step_2_cont}
        \exists \zb^* \in \Z, j \in [K], P \in C : \qquad  i \in P \implies \hat{I_j}(\hat\zb^*) \cap {I_i}(\zb^*)\neq \emptyset.
    \end{equation}
    We will let $\zb^*$ denote the value for which \cref{eqn:prop_2_step_2_cont} holds and without loss of generality let $j=1$.
    \paragraph{Step 2.1} 
    First, $\forall i \in P$ we define the sets:
    \begin{equation}\label{eqn:pixel_sets_def}
        O_{i,1} := \{\, j \in {I_i}(\zb^*) \mid j \in \hat {I_1}(\hat\zb^*) \,\}, \ \ O_{i,2} := \{\, j \in {I_i}(\zb^*) \mid j \notin \hat {I_1}(\hat\zb^*) \,\},
    \end{equation}
    Intuitively, the set $O_{i,1}$ represents the pixels which are a function of both ground-truth latent slot $\zb^*_{i}$ and inferred slot $\hat\zb^*_{1}$, while $O_{i,2}$ represents the pixels which are a function of $\zb^*_{i}$ but not $\hat\zb^*_{1}$. Our aim is now to show that for all $\forall i \in P$, the sets $O_{i,1}, O_{i,2}$ form a partition of ${I_i}(\zb^*)$.
    
    By \cref{eqn:pixel_sets_def}, $\forall i \in P$,  $O_{i,1}\cup O_{i,2} = {I_i}(\zb^*)$, and $O_{i,1}\cap O_{i,2} =\emptyset$. We thus only need to show that $O_{i,1},O_{i,2} \neq \emptyset$.

    We first note that by our assumed contradiction in \cref{eqn:prop_2_step_2_cont}, there are pixels which are a function of both ground-truth slot $\zb^*_{i}$ and inferred slot $\hat\zb^*_{1}$ i.e.:
    \begin{align}
        \forall i \in P, \exists j \in {I_i}(\zb^*) : j \in \hat {I_1}(\hat\zb^*) \implies j \in O_{i,1} \implies O_{i,1} \neq \emptyset.
    \end{align}    
    We will now show that $\forall i \in P,\ O_{i,2} \neq \emptyset$. Suppose for a contradiction that 
    \begin{equation} \label{eqn:prop_2_step_2.1_cont}
        \exists i \in P : O_{i,2} = \emptyset,
    \end{equation}    
    This implies that ${I_i}(\zb^*) = O_{i,1}$ as $ {I_i}(\zb^*) = O_{i,1}\cup O_{i,2} = O_{i,1}\cup \emptyset$. Further, \cref{eqn:pixel_sets_def} implies that $O_{i, 1} \subseteq \hat {I_1}(\hat\zb^*)$ thus $O_{i, 1}={I_i}(\zb^*) \subseteq \hat {I_1}(\hat\zb^*)$.
    
    Next, consider another ground-truth slot $\zb^*_{k}$ where $k \neq i \in P$. As previously established, $O_{k,1}\neq \emptyset$. Moreover, by \cref{eqn:pixel_sets_def}, $O_{k,1} \subseteq \hat {I_1}(\hat\zb^*)$. Thus, $A := {I_i}(\zb^*) \cup O_{k,1} \subseteq \hat {I_1}(\hat\zb^*)$. 
    Now, note that because $\hat\fb$ is compositional, \cref{lem:rank_upper_bound} implies that the rank of the sub-mechanism defined by $A \le M$. When coupled with the equality of Jacobian ranks between $\fb$ and $\hat\fb$ stated in \cref{lem:jacobian_rank_equality}, we get:
    \begin{equation}\label{eqn:prop_2_step_2.1_rank_ineq}
        \rank(\Jb\fb_{A}(\zb^*)) = \rank(\Jb\hat\fb_{A}(\hat \zb^*)) \le M.
    \end{equation}
    Moreover, according to \cref{eqn:pixel_sets_def}, $O_{k,1} \subseteq {I_k}(\zb^*)$. By compositionality of $\fb$, it thus follows that $O_{k,1} \cap {I_i}(\zb^*) = \emptyset$ since $i \neq k$. Therefore, by \cref{lem:independent_partition}, we know the sub-mechanisms defined by ${I_i}(\zb^*)$ and $O_{k,1}$ are independent such that
    \begin{align}
        \rank(\Jb\fb_{A}(\zb^*)) = \rank(\Jb\fb_{I_i}(\zb^*)) + \rank(\Jb\fb_{O_{k,1}}(\zb^*)).
    \end{align}
    Leveraging \cref{lem:compositional_rank} yields $\rank(\Jb\fb_{I_i}(\zb^*)) = M$. Inserting this in the previous equation yields
    \begin{align}
        \rank(\Jb\fb_{A}(\zb^*)) = M + \rank(\Jb\fb_{O_{k,1}}(\zb^*)),
    \end{align}
    which according to \cref{eqn:prop_2_step_2.1_rank_ineq} must be $\leq M$ i.e.
    \begin{align} \label{eqn:prop_2_step_2.1_rank_eq_ineq}
        M \geq \rank(\Jb\fb_{A}(\zb^*)) = M + \rank(\Jb\fb_{O_{k,1}}(\zb^*)).
    \end{align}
    Now, note that by the definition of ${I_k}(\zb^*)$ in \cref{eqn:compositional_index_sets}, $\forall i \in {I_k}(\zb^*)$, $\Jb\fb_i(\zb^*)\neq {\bf 0}$. Because $O_{k,1} \neq \emptyset$ and $O_{k,1} \subseteq {I_k}(\zb^*)$, it follows that $\Jb\fb_{O_{k,1}}(\zb^*) \neq {\bf 0}$. This implies $\rank(\Jb\fb_{O_{k,1}}(\zb^*)) > 0$. However, this contradicts \cref{eqn:prop_2_step_2.1_rank_eq_ineq} and, hence, also the initial assumption in \cref{eqn:prop_2_step_2.1_cont}. Therefore, we conclude that $\forall i \in P,\ O_{i,2} \neq \emptyset$.
    
    Taken together, we have shown that $\forall i \in P$, the sets $O_{i,1}, O_{i,2}$ are nonempty and form a partition of ${I_i}(\zb^*)$.
    \paragraph{Step 2.2} 
    Next, we first note that \cref{lem:compositional_rank} implies that the rank of the mechanism $\Jb\fb_{I_i}(\zb^*)$ is equal to $M$. Moreover, by assumption, $\Jb\fb_{I_i}(\zb^*)$ is irreducible. Because $O_{i,1}$ and $O_{i,2}$ form a partition of ${I_i}(\zb^*)$, irreducibility then implies:
    \begin{equation}\label{eqn:prop_2_step_2.2_rank_ineq_f}
        \forall i \in P : \rank(\Jb\fb_{O_{i,1}}(\zb^*)) + \rank(\Jb\fb_{O_{i,2}}(\zb^*)) > M.
    \end{equation}
     Due to the equality of Jacobian ranks between $\fb$ and $\hat\fb$ stated in \cref{lem:jacobian_rank_equality}, \cref{eqn:prop_2_step_2.2_rank_ineq_f} implies
    \begin{equation}\label{eqn:prop_2_step_2.2_rank_ineq_f_hat}
        \forall i \in P : \rank(\Jb\hat\fb_{O_{i,1}}(\hat\zb^*)) + \rank(\Jb\hat\fb_{O_{i,2}}(\hat\zb^*)) > M.
    \end{equation}
    By the definition of $O_{i,1},O_{i,2}$ in \cref{eqn:pixel_sets_def}, $\forall i \in P : O_{i,1} \subseteq \hat {I_1}(\hat\zb^*), \ O_{i,2} \cap \hat {I_1}(\hat\zb^*) = \emptyset$.
    It thus follows from \cref{lem:independent_partition} that the sub-mechanisms defined by $O_{i,1}$ and $O_{i,2}$ are independent under $\hat\fb$ in the sense of \cref{def:independent_dependent_submechanism}. Because $O_{i,1}$ and $O_{i,2}$ form a partition of ${I_i}(\zb^*)$, this independence, when coupled with \cref{eqn:prop_2_step_2.2_rank_ineq_f_hat}, implies:
    \begin{equation}\label{eqn:prop_2_step_2.2_rank_eq_ineq_f_hat}
        \forall i \in P : \rank(\Jb\hat\fb_{I_i}(\hat\zb^*)) =  \rank(\Jb\hat\fb_{O_{i,1}}(\hat\zb^*)) + \rank(\Jb\hat\fb_{O_{i,2}}(\hat\zb^*)) > M. 
    \end{equation}
    We know from \cref{lem:compositional_rank} that the mechanism defined by ${I_i}(\zb^*)$ has rank $M$ under $\fb$. The equality of Jacobian ranks between $\fb$ and $\hat\fb$ stated in \cref{lem:jacobian_rank_equality} then implies:
    \begin{align}
        \rank(\Jb\hat\fb_{I_i}(\hat\zb^*)) = \rank(\Jb\fb_{I_i}(\zb^*)) = M,
    \end{align}
    which contradicts \cref{eqn:prop_2_step_2.2_rank_eq_ineq_f_hat}, and, hence the initial assumption of this proof by contradiction in \cref{eqn:prop_2_step_2_cont} cannot be correct and \cref{eqn:prop_2_step_2_base} must hold true. 
    
    We have now shown that $\forall \zb\in \Z, j\in [K]$, there exists at least one and at most one $i\in [K] : \hat{I_j}(\hat\zb) \cap {I_i}(\zb) \neq \emptyset$ implying there exists exactly one, thus completing the proof.
\end{proof}
We now provide a corollary to \cref{prop:one_intersecting_object_indices} stating that the result also holds when the roles of $\hat{I_j}(\hat\zb), {I_i}(\zb)$ are reversed.
\vspace{\topsep}
\begin{corollary} \label{cor:rev_prop_2}
   $\forall \zb\in \Z, i\in [K]$, there exists exactly one $j\in [K] : \hat{I_j}(\hat\zb) \cap {I_i}(\zb) \neq \emptyset$.
\end{corollary}
\begin{proof}
We will first prove that there exists at least one $j\in [K] : \hat{I_j}(\hat\zb) \cap {I_i}(\zb) \neq \emptyset$. Assume, for a contradiction that:
    \begin{equation}\label{eqn:cor_1_step_1_cont}
        \exists \zb^*\in \Z, i \in [K], \nexists j \in [K]:\hat{I_j}(\hat\zb^*) \cap {I_i}(\zb^*) \neq \emptyset.
    \end{equation}
    This contradiction can be shown not to hold by exactly repeating the procedure in {\bf Step 1} of \cref{prop:one_intersecting_object_indices}.
        
We thus only need to prove that there exists at most one $j\in [K] : \hat{I_j}(\hat\zb) \cap {I_i}(\zb) \neq \emptyset$. Let $C:=\{\, P \subseteq [K] : |P| > 1 \,\}$. Suppose for a contradiction that:
    \begin{equation}\label{eqn:cor_1_step_2_cont}
        \exists \zb^* \in \Z, i \in [K], P \in C : \qquad  j \in P \implies \hat{I_j}(\hat\zb^*) \cap {I_i}(\zb^*)\neq \emptyset.
    \end{equation}

Let $A := [K] \setminus P$. We know by \cref{prop:one_intersecting_object_indices} that $\forall j \in A$, there exists exactly one $i \in [K] : \hat{I_j}(\hat\zb^*) \cap {I_i}(\zb^*)\neq \emptyset$. This implies that at least $|[K]| -|A| = |P|$ ground-truth latent slots generate pixels which do not overlap with the pixels generated by any inferred latent slots in $A$. In other words, there exists a set $B \subset [K]$ with cardinality $\geq |P| > 1$ s.t.
    \begin{equation}\label{eqn:cor_1_step_2_non_int_set_1}
        \forall i \in B, \forall j \in A : \hat{I_j}(\hat\zb^*) \cap {I_i}(\zb^*) = \emptyset
    \end{equation}
Now consider the set $P$. We know by \cref{eqn:cor_1_step_2_cont}, that for all $j \in P : \hat{I_j}(\hat\zb^*) \cap {I_i}(\zb^*)\neq \emptyset$. By \cref{prop:one_intersecting_object_indices}, we know that for all $j \in P$, $\hat{I_j}(\hat\zb^*)$ can intersect only with ${I_i}(\zb^*)$. Given that $|B| > 1$, this then implies
    \begin{equation}\label{eqn:cor_1_step_2_non_int_set_2}
        \exists i \in B : \forall j \in P : \hat{I_j}(\hat\zb^*) \cap {I_i}(\zb^*) = \emptyset
    \end{equation}
Now, by construction, $[K] = A \cup P$. Thus, \cref{eqn:cor_1_step_2_non_int_set_1} and \cref{eqn:cor_1_step_2_non_int_set_2} together imply:
    \begin{equation}\label{eqn:cor_1_step_2_non_int_set_union}
        \exists i \in B \subset [K] : \forall j \in [K] : \hat{I_j}(\hat\zb^*) \cap {I_i}(\zb^*) = \emptyset
    \end{equation}

We have already shown in the first part of this corollary, however, that \cref{eqn:cor_1_step_2_non_int_set_union} cannot be true by repeating the procedure in {\bf Step 1} of \cref{prop:one_intersecting_object_indices}. Thus, our assumed contradiction in \cref{eqn:cor_1_step_2_cont} cannot be true.

We have now shown that $\forall \zb\in \Z, i\in [K]$, there exists at least one and at most one $j\in [K] : \hat{I_j}(\hat\zb) \cap {I_i}(\zb) \neq \emptyset$ implying there exists exactly one, thus completing the proof.
\end{proof}
We now build upon~\cref{prop:one_intersecting_object_indices} and \cref{cor:rev_prop_2}, to show that all inferred latent slots depend on at most one ground-truth slot.
\begin{proposition} \label{prop:at_most_one_object}
     Let $\Z$ be a latent space and $\X$ an observation space. Let $\fb: \Z\rightarrow \X$ be a diffeomorphism that is compositional (\cref{def:compositional}) with irreducible mechanisms (\cref{def:irreducible_mechanism}). Let $\hat\gb: \X\rightarrow \Z$ be a diffeomorphism with inverse $\hat\fb: \Z\rightarrow \X$ that is compositional (\cref{def:compositional}). Let $\hat \zb := \hat\gb(\fb(\zb)), \forall \zb\in \Z$. Then, $\forall \zb\in \Z, i \in [K],$ there exists at most one $j \in [K] : \frac{\partial \hat \zb_j}{\partial \zb_i}(\zb) \neq {\bf 0}$.
\end{proposition}
\begin{proof}
    Our goal is to show that at most one $\hat \zb_j$ is a function of a given $\zb_i$. More precisely, let $C:=\{\, P \subseteq [K] : |P| > 1 \,\}$. We aim to show that:
        \begin{equation}\label{eqn:prop_3_base}
            \forall \zb\in \Z, i \in [K], \nexists P \in C: \qquad  j \in P \implies  \frac{\partial \hat \zb_j}{\partial \zb_i}(\zb) \neq {\bf 0}.
        \end{equation}
    Suppose for a contradiction to \cref{eqn:prop_3_base} that:
    \begin{equation}\label{eqn:prop_3_cont}
        \exists \zb^* \in \Z, i \in [K], P \in C : \qquad  j \in P \implies  \frac{\partial \hat \zb_j}{\partial \zb_i}(\zb^*) \neq {\bf 0}.
    \end{equation}
    Let $\zb^*$ denote the value for which \cref{eqn:prop_3_cont} holds and without loss of generality let $i=1$.

    We first introduce the function
        \begin{align*}
            \hb := \hat\gb \circ \fb: \Z\rightarrow \Z \textrm{ s.t. } \hat \zb := \hat\gb (\fb(\zb)) = \hb(\zb).
        \end{align*}
    
    Note that $\fb = \hat\fb \circ \hat\gb \circ \fb = \hat\fb \circ \hb$. Thus, $\forall S \subseteq [N], \ \ \fb_{S} = \hat\fb_{S} \circ \hb$.
        Therefore,
        \begin{equation}\label{eqn:prop_3_jac_eq}
            \forall \zb \in \Z, j \in [K] : \frac{\partial \fb_{{\hat I}_j}}{\partial \zb_1}(\zb) 
            = \frac{\partial 
            \hat \fb_{{\hat I}_j}}{\partial \hat \zb}
            (\hat{\zb})
            \frac{\partial \hat \zb}{\partial \zb_1}(\zb)
        \end{equation}
    Due to the compositionality of $\hat \fb$, $\frac{\partial 
            \hat \fb_{{\hat I}_j}}{\partial \hat \zb_k} 
            (\hat{\zb}) = {\bf 0}, \forall k \neq j \in [K]$. This implies that these columns can be ignored when taking the product in \cref{eqn:prop_3_jac_eq}, s.t.
         \begin{equation}\label{eqn:prop_3_jac_eq_comp}
            \frac{\partial 
            \hat \fb_{{\hat I}_j}}{\partial \hat \zb}
            (\hat{\zb})
            \frac{\partial \hat \zb}{\partial \zb_1}(\zb) = \frac{\partial 
                    \hat \fb_{{\hat I}_j}}{\partial \hat \zb_j} 
                    (\hat{\zb})
                    \frac{\partial \hat \zb_j}{\partial \zb_1}(\zb).
                \end{equation}
        Now by \cref{cor:rev_prop_2}, there exists exactly one $j\in P \subseteq [K]$ s.t. $\hat{I_j}(\hat\zb^*) \cap {I_1}(\zb^*) \neq \emptyset$. By the definition of ${I_i}(\zb)$ in \cref{eqn:compositional_index_sets}, this implies that there exists exactly one $j\in P$ s.t. $\frac{\partial \fb_{{\hat I}_j}}{\partial \zb_1}(\zb^*) \neq {\bf 0}$. $|P| > 1$, thus there exists a $j\in P$ s.t.
        \begin{equation}\label{eqn:prop_3_zero_jac}
        \frac{\partial \fb_{{\hat I}_j}}{\partial \zb_1}(\zb^*) = \frac{\partial 
            \hat \fb_{{\hat I}_j}}{\partial \hat \zb_j}(\hat{\zb}^*)\frac{\partial \hat \zb_j}{\partial \zb_1}(\zb^*) = {\bf 0}
        \end{equation}
        where we leveraged \cref{{eqn:prop_3_jac_eq}}, \cref{eqn:prop_3_jac_eq_comp} to get the first equality above. Now, we know by \cref{lem:compositional_rank}, that $\Jb\hat\fb_{{\hat I}_j}(\hat \zb^*)$ is full column-rank. By compositionality of $\hat\fb$, we also know that $\rank(\Jb\hat\fb_{{\hat I}_j}(\hat \zb^*)) = \rank(\frac{\partial \hat\fb_{{\hat I}_j}}{\partial \hat \zb_j}(\hat \zb^*))$ as these are the only non-zero columns in $\Jb\hat\fb_{{\hat I}_j}(\hat \zb^*)$. Thus, $\frac{\partial \hat\fb_{{\hat I}_j}}{\partial \hat \zb_j}(\hat \zb^*)$ is also full column-rank. Now, \cref{eqn:prop_3_zero_jac} implies that all columns of $\frac{\partial \hat \zb_j}{\partial \zb_1}(\zb^*)$ must be in $\mathrm{null}(\frac{\partial 
            \hat \fb_{{\hat I}_j}}{\partial \hat \zb_j}(\hat \zb^*))$. Because, $\frac{\partial 
            \hat \fb_{{\hat I}_j}}{\partial \hat \zb_j}(\hat \zb^*)$ is full-column rank, $\mathrm{null}(\frac{\partial 
            \hat \fb_{{\hat I}_j}}{\partial \hat \zb_j}(\hat \zb^*))={\bf 0}$. However, by \cref{eqn:prop_3_cont} at least one column of $\frac{\partial \hat \zb_j}{\partial \zb_1}(\zb^*)$ is non-zero. Thus, we obtain a contradiction and conclude that \cref{eqn:prop_3_base} must hold.
\end{proof}
\vspace{\topsep}

Building on top of the previous propositions, we now prove our main identifiability result:
\identifiability*
\begin{proof}\vspace{-1.8\topsep}
    According to \cref{prop:block_identifiability_at_least_one_slot} every inferred latent slot $\hat\zb_j$ depends on \textit{at least} one ground-truth latent slot $\zb_i$. At the same time, \cref{prop:at_most_one_object} states that every inferred latent slot depends on \textit{at most} one ground-truth slot. Hence, every inferred latent slot depends on \textit{exactly} one ground-truth slot.
    
    This implies that the Jacobian $\Jb\hb(\zb)$ of $\hb= \hat\gb \circ \fb: \Z\rightarrow \Z$ must be block diagonal up to permutation everywhere:
    \begin{align}
        \forall \zb \in \Z: \qquad \Jb\hb(\zb)=\Pb(\zb)\Bb(\zb)
    \end{align}
    where $\Pb(\zb)$ is a permutation matrix and $\Bb(\zb)$ a block-diagonal matrix.

Next, note that
\begin{equation}
\det(\Jb\hb(\zb))=\det(\Pb(\zb))\det(\Bb(\zb))=\det(\Bb(\zb))\neq 0 \qquad 
\end{equation}
since $\hb$ is diffeomorphic.
Hence, $\Bb(\zb)$ is invertible with continuous inverse. We conclude that \begin{align}
    \Pb(\zb)=\Jb\hb(\zb)\Bb^{-1}(\zb)
\end{align}
is continuous. At the same time, $\Pb(\zb)$ can only attain a finite set of values since it is a permutation. Hence, $\Pb(\zb)$ must be constant in $\zb$, that is, the same global permutation is used everywhere.\footnote{Suppose for a contradiction that $\Pb(\zb)$ attains distinct values at some $\zb^A\neq\zb^B$ in $\Zcal$. Since $\Zcal$ is convex, the line connecting $\zb^A$ and $\zb^B$ is also in $\Zcal$ and $\Pb$ must change value somewhere along this line, leading to a discontinuity and thus a contradiction.}

    Thus, for any $j\in K$, there exists a \textit{unique} $i\in K$ such that the function $\hb_{j}= \hat\gb_{j}\circ \fb : \Z\rightarrow \Z_j$ is, in fact, constant in all slots except  $\Zcal_i$, i.e., it can be written as a mapping $\hb_{j} : \Z_i\rightarrow \Z_j$.
    
    Finally, all such $\hb_{j}$ are diffeomorphic, since $\hb$ is a diffeomorphism.
    
    This concludes the proof that assumptions \emph{(i)} and \emph{(ii)} imply $\hat\gb$ slot-identifies $\zb$.
\end{proof}

We now show that the compositional contrast $\Ccomp$ introduced in \cref{eq:compositional_contrast} indicates whether a map is compositional:
\vspace{\topsep}
\begin{lemma} \label{lem:compositional_contrast}
    Let $\fb: \Z\rightarrow \X$ be a differentiable function. $\fb$ is compositional in the sense of \cref{def:compositional} if and only if  for all~$\zb \in \Z$: $$\Ccomp(\fb, \zb) = 0\,.$$
\end{lemma}
\begin{proof}\vspace{-1.8\topsep}

    $(\Rightarrow)$ We begin by analyzing $\Ccomp(\fb, \zb)$:
    \begin{equation}\label{eq:c_comp_repeat}
        \sum_{n=1}^{N} \sum_{k=1}^{K} \sum_{j=k+1}^{K} \left\|\frac{\partial f_{n}}{\partial \zb_{k}}(\zb)\right\|_{2} \left\|\frac{\partial f_{n}}{\partial \zb_{j}}(\zb)\right\|_{2}
    \end{equation}
    Since all summands are non-negative, the sum can only equal zero if every summand is zero $\forall \zb \in \Z$. Since $j\neq k$ in the summand, this means:
        \begin{equation}
            \forall \zb \in \Z, \forall n \in [N], k\neq j \in [K]: \left\|\frac{\partial f_{n}}{\partial \zb_{k}}(\zb)\right\|_{2} \left\|\frac{\partial f_{n}}{\partial \zb_{j}}(\zb)\right\|_{2} = 0 \\
        \end{equation}
    This relation can only be satisfied if one (or both) of the partial derivatives in the summand have a norm of zero, i.e. if they are zero. More precisely,
        \begin{equation}\label{eq:lemma_compositional_contrast_loss_simplified}
            \forall \zb \in \Z, \forall n \in [N],  k\neq j \in [K]: \frac{\partial f_{n}}{\partial \zb_{k}}(\zb) = {\bf 0} \lor \frac{\partial f_{n}}{\partial \zb_{j}}(\zb) = {\bf 0}.
        \end{equation}
  
    According to~\cref{def:compositional} a map $\fb$ is compositional if 
        \begin{equation}\label{eq:repeat_comp}
            \forall \zb\in \Z: \qquad   k\neq j \implies   I_k(\zb) \cap I_j(\zb) = \emptyset.
        \end{equation}
    By the definition of $I_i(\zb)$ in \cref{eqn:compositional_index_sets}, we can restate \cref{eq:repeat_comp} as:
        \begin{equation}\label{eq:restate_comp}
            \forall \zb\in \Z, k\neq j, \nexists n \in [N] : \frac{\partial f_{n}}{\partial \zb_{k}}(\zb) \neq {\bf 0} \land \frac{\partial f_{n}}{\partial \zb_{j}}(\zb) \neq {\bf 0}
        \end{equation}
    which implies:
        \begin{equation}\label{eq:imp_comp}
            \forall \zb\in \Z, n \in [N], k\neq j : \frac{\partial f_{n}}{\partial \zb_{k}}(\zb) = {\bf 0} \lor \frac{\partial f_{n}}{\partial \zb_{j}}(\zb) = {\bf 0}
        \end{equation}
    which is equivalent to \cref{eq:lemma_compositional_contrast_loss_simplified}. Hence, $\forall \zb \in \Z: \Ccomp(\fb, \zb) = 0$ implies that $\fb$ is compositional.

    $(\Leftarrow)$ We now prove the reverse direction i.e. that if $\fb$ is compositional, then $\forall \zb \in \Z: \Ccomp(\fb, \zb) = 0$. Note that the form of compositionality given in \cref{eq:restate_comp} implies that $\forall \zb \in \Z$, at least one term in the summand of $\Ccomp(\fb, \zb)$ in \cref{eq:imp_comp} will be zero. Thus, each summand is equal to zero. This then implies that $\forall \zb \in \Z: \Ccomp(\fb, \zb) = 0$, completing the proof.
\end{proof}

Finally, by leveraging \cref{lem:compositional_contrast}, we can obtain \cref{thm:slot_identifiability} in a less abstract form.
\vspace{\topsep}
\contrast*
\begin{proof}\vspace{-1.8\topsep}
    As both summands of the functional are non-negative, solving the functional equation means solving for each of the summands to be equal to zero. Thus, we can analyze both of them separately. Solving the first sub-functional equation, i.e.,
    \begin{align*}
        \mathbb{E}_{\xx \sim p_\xx}\left[\left\|\hat\fb(\hat \gb(\xb))-\xb\right\|^{2}_{2}\right] = 0,
    \end{align*}
    implies that $\hat\fb$ is an inverse of $\hat\gb$ for every $\xb \sim p_\xb$. Because $p_\zb$ is assumed to have full support over $\Z$, and $p_\xb$ is defined by applying a diffeomorphism $\fb: \Z\rightarrow \X$ on $p_\zb$, this implies that $p_\xb$ has full support over $\X$. This means that $\hat\fb$ is an inverse of $\hat\gb$ over the entire space $\X$ i.e. $\hat\fb = \hat\gb^{-1}$. Since per assumption $\hat\gb$ and $\hat\fb$ are differentiable it follows that $\hat\gb$ is a diffeomorphism. 
    
    Moreover, per~\cref{lem:compositional_contrast}, solving the second sub-functional equation for $\lambda > 0$, i.e.,
    \begin{align*}
        \mathbb{E}_{\xx \sim p_\xx}\left[\lambda\Ccomp(\hat \fb, \hat\gb(\xb))\right] = 0,
    \end{align*}
    means that ${\hat\fb}$ is compositional as $p_\xb$ has full support over $\X$ and $\hat\g$ is a diffeomorphism between $\X$ and $\Z$. From \cref{thm:slot_identifiability} it now follows that $\hat\gb$ slot-identifies $\zb$, concluding the proof.
\end{proof}

\section{Experimental Details}
\label{sec:experimental_details}
\subsection{Synthetic Data \cref{subsec:exp1}}\label{subsec:exp1_details}
\paragraph{Enforcing Irreducibility}
We choose  slot-output dimension, which we will denote $\mathrm{dim}(\xb_{s})$, to be greater than slot-dimension $M$ as this is required for irreducibility~(\cref{def:irreducible_mechanism}). To see this, assume the number of rows in each mechanism (\cref{def:mechanism}), equal in our case to $\mathrm{dim}(\xb_{s})$, were equal to $M$. Because mechanisms have $\rank = M$~(\cref{lem:compositional_rank}) and we have $M$ rows, this implies that no row is in the span of any others. Hence, the mechanism would be reducible. Beyond enforcing that the slot-output dimension, equal to $20$ in this case, is greater than $M=3$, we do not do anything further to ensure that our ground-truth generator is irreducible. This is because it is extremely unlikely that the generator, as we have constructed it, could be reducible. Specifically, if the generator were reducible, then as $\mathrm{dim}(\xb_{s})$ becomes larger than $M$, each new row in the Jacobian would need to lie in the span of some subset of the previous rows. As $\mathrm{dim}(\xb_{s})$ continues to increase relative to $M$, however, this becomes increasingly unlikely since the rows in the weight matrices of our MLP generator are randomly sampled i.e. entries are sampled uniformly from $[-10, 10]$.
\paragraph{Inference Model Training and Evaluation}
For our inference model, we use a 3 layer MLP with 80 hidden units in each layer and LeakyReLU activation functions. We train on 75,000 samples and use 6,000 and 5,000 for validation and test sets, respectively. We train for 100 epochs with the Adam optimizer~\citep{kingma2014adam} on batches of 64 with an initial learning rate of $10^{-3}$, which we decay by factor of 10 after 50 epochs. We use the validation set to find the optimal permutation for the Hungarian matching and then evaluate the SIS on the test set after applying this permutation to the slots. We compute the SIS for models every 4 epochs during training, all of which are plotted in~\cref{fig:results_toy_image_experiments}. We trained all models using PyTorch~\citep{Paszke2019PyTorchAI}.
\subsection{Existing Object-Centric Models \cref{subsec:exp2}}\label{subsec:exp2_details}
\paragraph{Data Generation}
We generate image data using the Spriteworld renderer \citep{spriteworld19}. Images consist of $2$ to $4$ objects, each described by $4$ continuous (size, color, x/y position) and $1$ discrete (shape) independent latent factors. We sample all factors uniformly where size is sampled from $[.1, .15]$ and x/y position both from $[.1, .8]$. We represent color using HSV and sample hue from $[0, 1]$ while fixing saturation and value to 3 and 1, respectively. The dataset consists of 100,000 images, 90,000 of which are used for training and 10,000 for evaluation.
\paragraph{Inference Model Training and Evaluation}
We use the same Slot Attention model proposed by~\citet{locatello2020object}, with the changes being that we use $16$ convolutional filters in the decoder opposed to $32$ and do not use a learning rate warm-up. For MONet, we follow the setup used by~\citet{dittadi2021generalization} on Multi-dSprites~\citep{multiobjectdatasets19}. For our additive autoencoder, we use the convolutional encoder/decoder architecture proposed by~\citet{burgess2018understanding}. The model decodes each slot separately to get slot-wise reconstructions and mask, applies the normalized mask to each slot-wise reconstruction, and then adds the results together to get the final reconstructed image. For all models, we use $4$ slots with a slot-dimension of $16$. We train all models for $500,000$ iterations ($356$ epochs) on batches of $64$ with between $5$ to $12$ random seeds for each model. We train using the Adam optimizer~\citep{kingma2014adam} with an initial learning rate of $10^{-4}$, which we decay throughout training for all models using the same decay scheduler as~\citet{locatello2020object}. We trained all models using PyTorch~\citep{Paszke2019PyTorchAI}.

\subsection{Compositional Contrast Normalized Variants}\label{subsec:c_comp_ext}
When computing $\Ccomp$ in \cref{subsec:exp1} and \cref{subsec:exp2}, we use a few different normalized variants of the contrast to overcome potential issues with the definition given in~\cref{def:compositional_contrast}. Firstly, as the number of latent slots $K$ increases, the contrast in~\cref{def:compositional_contrast} will scale by a factor $K^{2}-K$. Thus, when comparing models across different numbers of slots in \cref{subsec:exp1}, we divide the contrast by this factor to ensure that comparisons remain meaningful across different values of $K$. Another issue with the contrast in~\cref{def:compositional_contrast}, is that it is not scale invariant. Specifically, naively minimizing the norm of the gradients for each pixel across slots will also minimize the contrast, despite all slots having similar gradient norms for a given pixel. This scale invariance did not cause issues when optimizing $\Ccomp$ directly in \cref{subsec:exp1}. However, when evaluating the $\Ccomp$ of object-centric models in \cref{subsec:exp2}, we account for this invariance. Specifically, we divide the gradient norms for each pixel with respect to each slot by the mean gradient norm for this pixel across slots. This gradient normalization creates an additional problem, however: Pixels with a relatively small gradient norm, such as black background pixels, will be weighted equally to pixels with a larger gradient norm such as pixels corresponding to an object. To account for this, we weight each pixel's contribution to the contrast by the pixel's mean gradient across slots.

\subsection{Slot Identifiability Score}\label{subsec:sis_details}
We are interested in a metric measuring how much information about the ground-truth latent slots is contained in the inferred latent slots without mixing information about different ground-truth slots into the same inferred slot. Let $S_1, S_2 \in [0, 1]$ denote scores that quantify how much information about each ground-truth slot can be extracted from the most and second-most predictive inferred slot, respectively.
The aforementioned metric can be computed by just subtracting the two scores, i.e.
\begin{align}
    S = S_1 - S_2.
\end{align}
Following previous work, we use the $R^{2}$ coefficient of determination as a score for continuous factors of variation (which we restrict to be strictly non-negative) and the accuracy for categorical factors \citep{dittadi2021generalization}. We compute one $S$ value for each type and take the weighted mean which we then average across all slots to get the final slot identifiability score (SIS).

\paragraph{Computing SIS on Synthetic Data \cref{subsec:exp1}}
To compute the scores $S_1$ and $S_2$ defined in our experiments in \cref{subsec:exp1}, we must fit two inference models between ground-truth and inferred slots: one between the best-matching slots and one between the second-best-matching slots. In~\cref{subsec:exp1}, we fit these models by first fitting a kernel ridge regression model between every pair of inferred and ground-truth slots and computing the $R^{2}$ scores for the predictions given by each model. We then use the Hungarian algorithm~\citep{kuhn1955hungarian} to match each ground-truth slot to its most predictive inferred slot based on these $R^{2}$ scores, which gives us $S_1$. To get $S_2$, we take the highest $R^{2}$ score for each inferred slot with respect to the ground-truth slots that it was not already matched with. For our experiments in~\cref{fig:toy_dependent_latents} with dependent latent slots, $S_2$ will inevitably be non-zero even if a model is perfectly identifiable. Thus, for these experiments, we only consider $S_1$ and refer to this metric as the Slot MCC (Mean Correlation Coefficient).
\paragraph{Computing SIS on Image Data \cref{subsec:exp2}}
When training models to compute $S_1$ and $S_2$ in our experiments on image data in~\cref{subsec:exp2}, one issue that arises is that the permutation between inferred latent slots and ground-truth slots is not necessarily a global permutation but can also be a local permutation. This is due to the ground-truth generator function being permutation invariant. To resolve this, we take a similar approach to work by~\citet{dittadi2021generalization} and perform an online matching during training of inferred latent slots to ground-truth slots using the training loss. Specifically, we compute the loss for every pairing of the ground-truth and inferred slots and use the Hungarian algorithm to pick the permutation that yields the lowest aggregate loss. As every slot can contain both continuous and categorical variables, we compute the mean squared error for continuous factors and cross-entropy for categorical variables and sum them up to obtain the training loss. In our experiments, we notice that the cross-entropy tends to yield unstable matching results. Therefore, we use the minimum probability margin \footnote{i.e., $\max_i p_i - p_y$, where $p$ denotes the predicted probability for different values of the categorical distribution and $y$ the ground-truth value} to compute the categorical loss to solve the matching problem.
Before fitting the readout models, we standardized both the ground-truth and inferred latents. We parameterized the readout models as $5$-layer MLPs with LeakyReLU nonlinearity and a hidden dimensionality of $256$, and trained them for up to $100$ epochs using the Lion optimizer with a learning rate of $10^{-4}$. To prevent the network from locking in too early on a suboptimal solution, we add a small amount of noise ($10$\,\% of the maximum matching loss value) to the losses before determining the optimal matching. Finally, we suggest performing cross-validation and early stopping to prevent overfitting.

For training the model to compute $S_2$, we proceed as for $S_1$ but ensure that the model is not using the same permutation used for computing $S_1$, i.e., it is trained on the second-best matching between ground-truth and inferred slots. Lastly, when computing $S_2$, we aim to avoid scenarios in which the model finds a spurious permutation yielding a non-zero $S_2$ despite the model being identifiable. To account for this, we compute $S_2$ on the ground-truth latent slots, denoted $S_2^{\textrm{gt}}$, using the same procedure for computing $S_2$, and use this score to adjust our previous scores. Specifically, by adjusting the value range accordingly, we obtain a score of
\begin{align}
    S = \frac{S_1 - S_2^{\textrm{gt}}}{1 - S_2^{\textrm{gt}}} - \frac{S_2 - S_2^{\textrm{gt}}}{1 - S_2^{\textrm{gt}}},
\end{align}
To ensure that the subtracting term is not increasing the final score, we restrict it to be positive, yielding the final score:
\begin{align} \label{eq:appx_sis}
    S = \frac{S_1 - S_2^{\textrm{gt}}}{1 - S_2^{\textrm{gt}}} - \max\left( \frac{S_2 - S_2^{\textrm{gt}}}{1 - S_2^{\textrm{gt}}}, 0\right).
\end{align}
We may additionally be interested in considering the two terms on the RHS of~\cref{eq:appx_sis} separately. Thus, we define them below as:
\begin{align}
    \hat S_1 = \frac{S_1 - S_2^{\textrm{gt}}}{1 - S_2^{\textrm{gt}}}, \quad
    \hat S_2 = \frac{S_2 - S_2^{\textrm{gt}}}{1 - S_2^{\textrm{gt}}}, \quad
    S = \hat S_1 - \max(\hat S_2, 0).
\end{align}

\clearpage
\section{Additional Figures and Experiments}\label{sec:addit_figures}
\begin{figure}[tbh]
    \centering
    \includegraphics[width=0.5\linewidth]{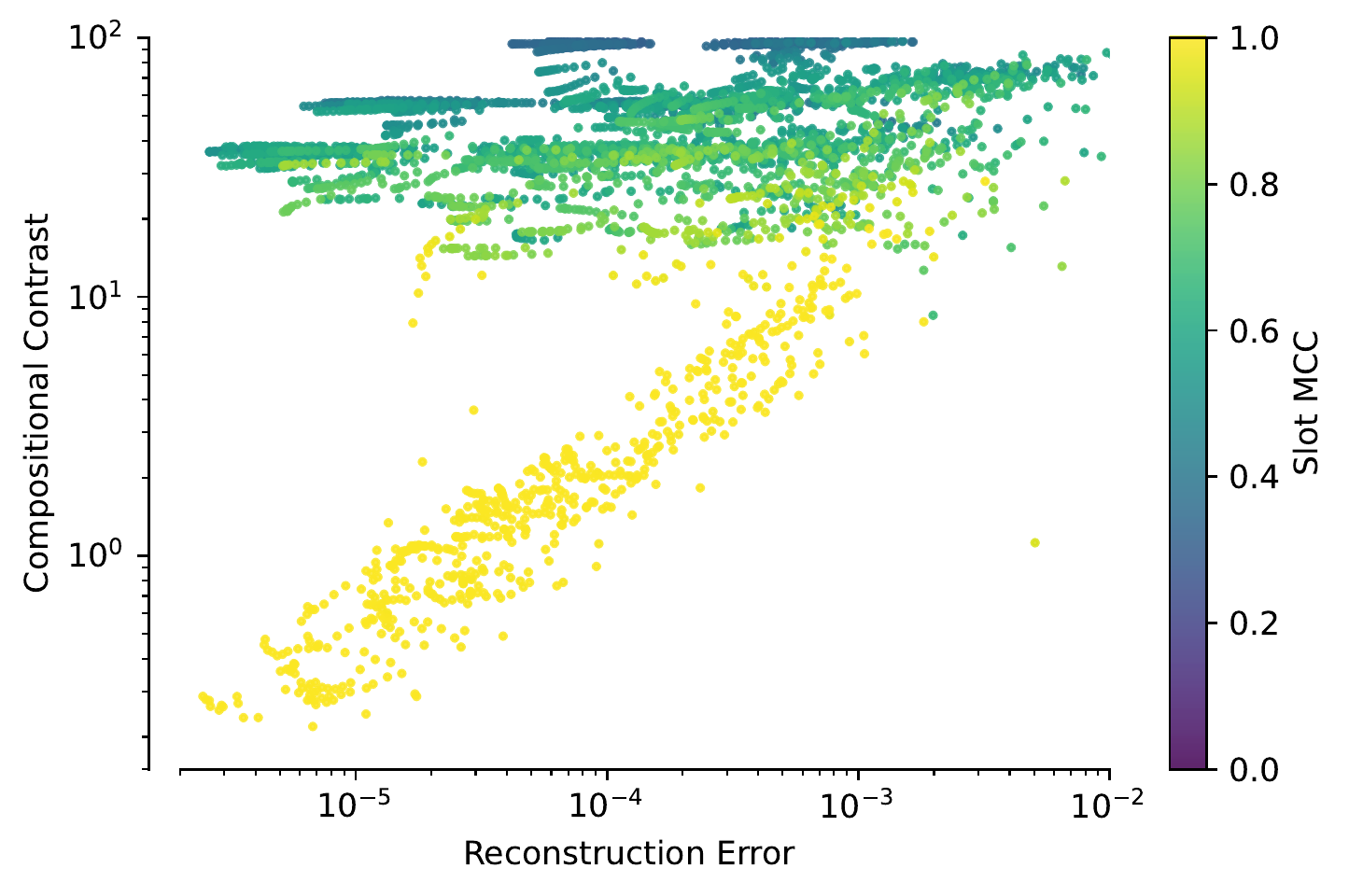}
    \caption{\textbf{Experimental validation of \cref{thm:compositional_contrast_identifiable} for statistically dependent slots.} We trained models on synthetic data generated according to \cref{sec:model} with 2, 3, 5 dependent latent slots (see~\cref{subsec:exp1}). The color coding indicates the level of identifiability achieved by the model, measured by the Slot Mean Correlation Coefficient (MCC), where higher values correspond to more identifiable models. As predicted by our theory, if a model sufficiently minimizes both reconstruction error and compositional contrast, then it identifies the ground-truth latent slots.}
    \label{fig:toy_dependent_latents}
\end{figure}

\begin{figure}[tbh]
    \centering
    \includegraphics[width=0.45\linewidth]{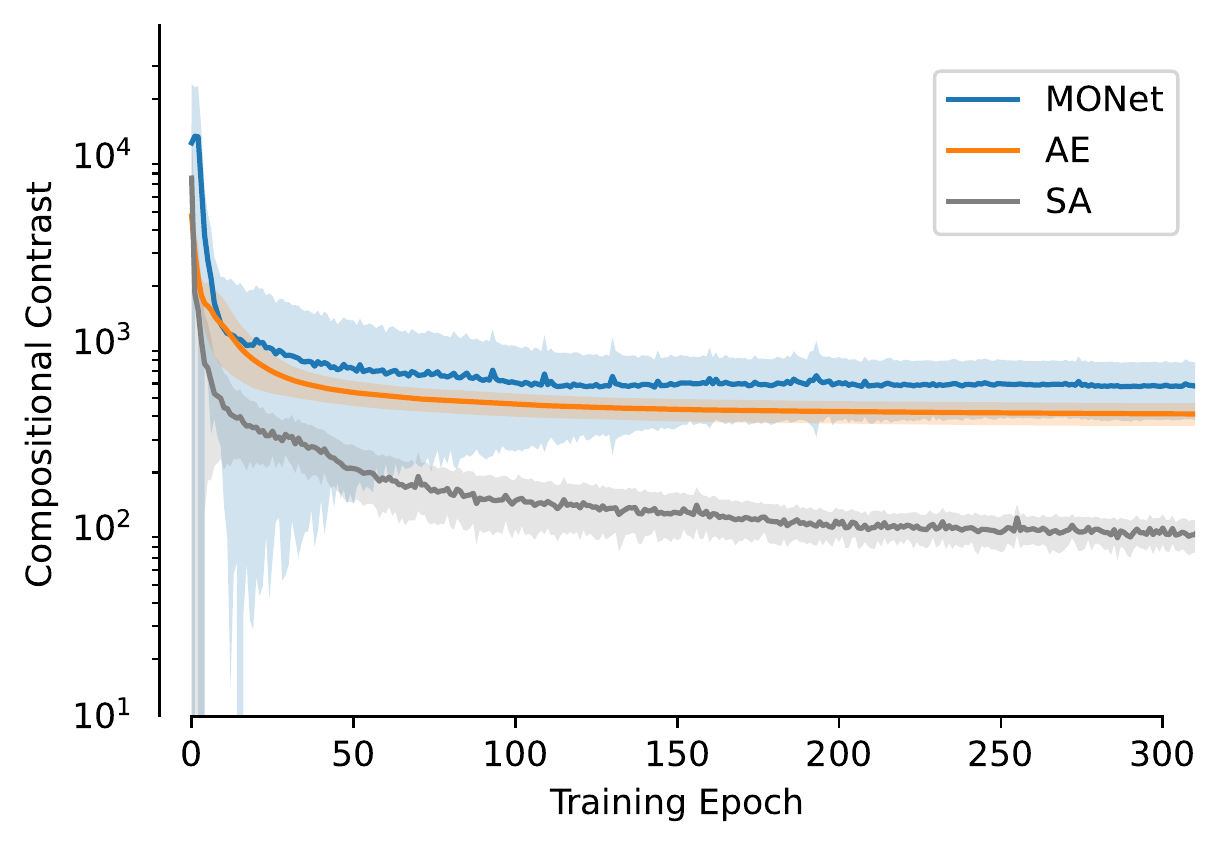}
    \caption{\textbf{Compositional Contrast ($\Ccomp$) throughout training.} Here, we plot the compositional contrast ($\Ccomp$) over the course of training for MONet, Slot Attention (SA) as well as an additive auto-encoder (AE), on image data. We can see that all models appear to be minimizing $\Ccomp$ to some extent despite it not being explicitly optimized for in any of these models.}
    \label{fig:image_models_training_curves}
\end{figure}

\begin{figure*}[tbh]
    \centering
    \includegraphics[width=\linewidth]{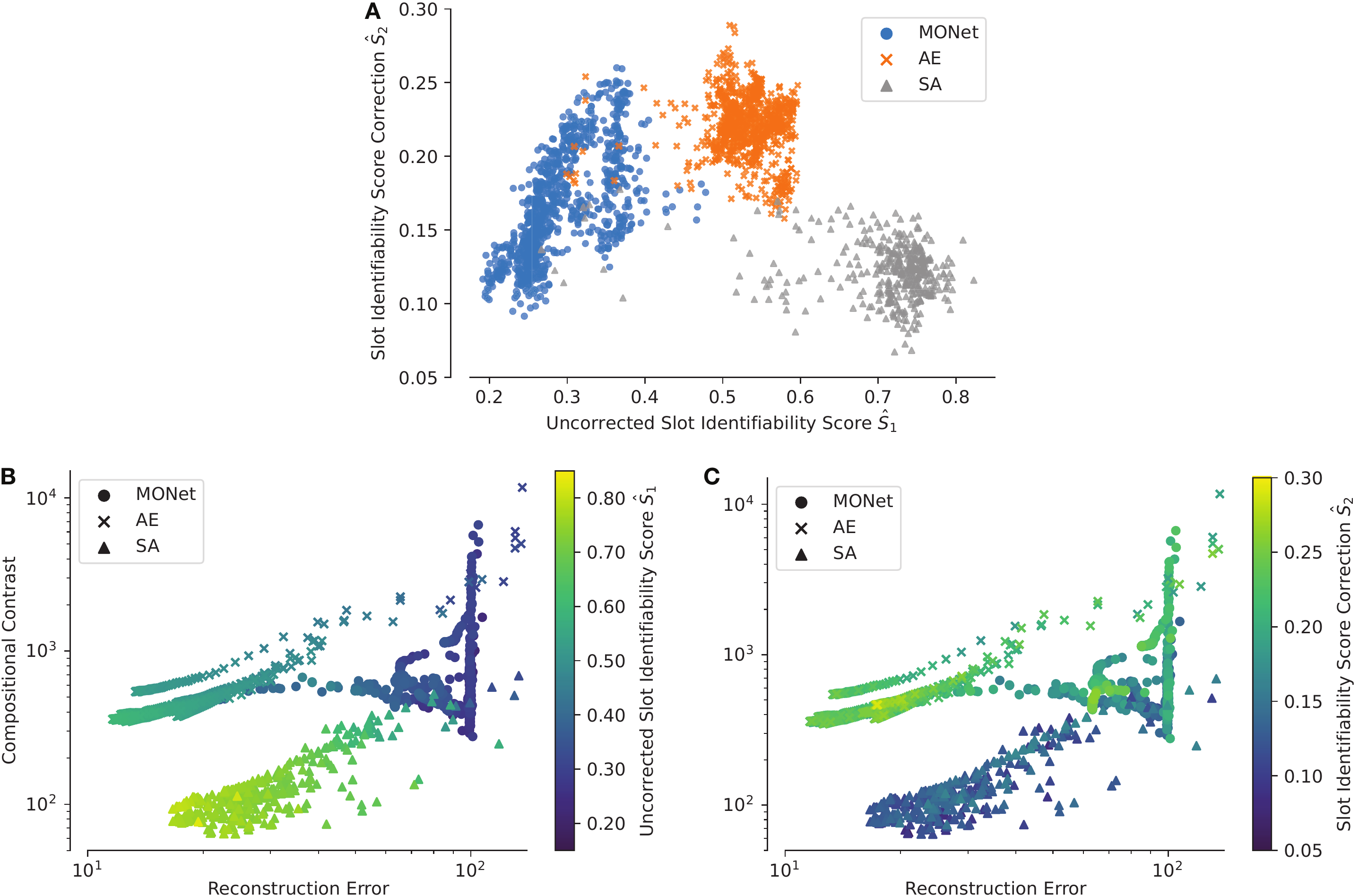}
    \vspace{-2em}
    \caption{\textbf{Analysis of Information Leakage Between Slots from Models Trained in \cref{subsec:exp2}.} 
    \textbf{(A) Uncorrected Slot Identifiability Score ($\hat S_1$) vs. Correction ($\hat S_2$).}
    We train 3 existing object-centric architectures---MONet, Slot Attention (SA), and an additive auto-encoder (AE)---on image data and investigate whether inferred latent slots encode information from multiple objects when using an inferred latent dimension greater than the ground-truth. To test this, we look at the $R^2$ score for a model fit between each inferred slot and the second most predictive ground-truth slot for this slot. We refer to this score as the \emph{slot identifiability score correction}, defined as $\hat S_2$ in \cref{subsec:sis_details}. We plot this score against the uncorrected slot identifiability score i.e. the most predictive ground-truth slot, defined as $\hat S_1$ in \cref{subsec:sis_details}. We can see that for all models, $\hat S_2$ is non-zero, even as $\hat S_1$ increases, suggesting that models are leveraging their additional latent capacity to encode information about multiple objects in the same latent slot.
    \textbf{(B) and (C) Influence of Reconstruction Error and Compositional Contrast on $\hat S_1$ and $\hat S_2$.}
    Here, we further visualize the slot identifiability score correction ($\hat S_1$) and the uncorrected score ($\hat S_2$) as a function of the reconstruction error and the compositional contrast in panels B and C, respectively. 
    We can see in B that, similar to the SIS in~\cref{fig:results_toy_image_experiments}, $\hat S_1$ tends to increase as reconstruction loss and compositional contrast decrease. We can additionally see in C that, while $\hat S_2$ decreases to some extent with $\Ccomp$, there is generally less of a correlation between $\hat S_2$ and these metrics. This suggests that the latent capacity must also be restricted to minimize $\hat S_2$.}
\label{fig:joint_results_spriteworld_on_off_and_on_vs_off_and_on_and_off}
\end{figure*}

\begin{figure*}[tbh]
    \centering
    \includegraphics[width=0.8\linewidth]{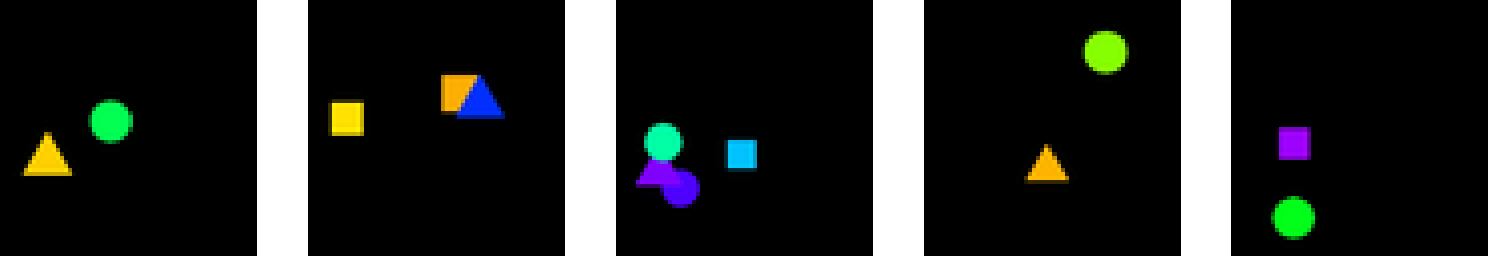}
    \caption{Samples from our multi-sprites dataset used in ~\cref{subsec:exp2}. Objects are described by five latent factors: shape, color, size, and x/y position. Occlusions are present in the dataset, as shown in the samples above (see the second and third images from the left).}
    \label{fig:sprites_sample}
\end{figure*}
\end{document}